\documentclass[aoas, preprint]{imsart}
\usepackage{amsmath,amssymb,amsthm}
\usepackage{graphicx,psfrag,epsf}
\usepackage{enumerate}
\usepackage{natbib}
\usepackage{url} % not crucial - just used below for the URL  
\usepackage{bm}

\RequirePackage[colorlinks,citecolor=blue,urlcolor=blue]{hyperref}

\doi{10.1214/21-AOAS1495}
\pubyear{2022}
\volume{16}
\issue{1}
\firstpage{193}
\lastpage{215}

\def\U{\mathbf{U}}

\def\X{\mathbf{X}}

\def\E{\mathbf{E}}

\def\V{\mathbf{V}}
\def\A{\mathbf{A}}
\def\Z{\mathbf{Z}}

\def\M{\mathbf{M}}
\def\N{\mathbf{N}}

\def\R{\mathbf{R}}

\def\C{\mathbf{C}}

\def\D{\mathbf{D}}

\def\S{\mathbf{S}}
\def\0{\mathbf{0}}

\def\min{\mbox{min}}
\def\Sol{\{\S_{\bigcdot \bigcdot}^{(k)}\}_{k=1}^K}
\def\SolHat{\{\hat{\S}_{\bigcdot \bigcdot}^{(k)}\}_{k=1}^K}
\def\SolTil{\{\tilde{\S}_{\bigcdot \bigcdot}^{(k)}\}_{k=1}^K}
\def\fpen{f_{\text{pen}}}
\def\SolSpace{\mathbb{S}_{\hat{\X}}}
\def\Sk{\S_{\bigcdot \bigcdot}^{(k)}}
\def\hatSk{\hat{\S}_{\bigcdot \bigcdot}^{(k)}}
\def\tilSk{\tilde{\S}_{\bigcdot \bigcdot}^{(k)}}

\newcommand{\bC}{ {\bf C} }

\newcommand{\bR}{ {\bf R} }
\newcommand{\bS}{ {\bf S} }

\newcommand{\bzero}{ {\bf 0} }

\makeatletter
\newcommand*\bigcdot{\mathpalette\bigcdot@{.5}}
\newcommand*\bigcdot@[2]{\mathbin{\vcenter{\hbox{\scalebox{#2}{$\m@th#1\bullet$}}}}}
\makeatother

\newtheorem{prop}{Proposition} 
 
\newtheorem{theorem}{Theorem}

\newtheorem{lemma}{Lemma}

\begin{document}

\begin{frontmatter}
\title{Bidimensional linked matrix factorization for pan-omics pan-cancer analysis}
%\title{A sample article title with some additional note\thanksref{t1}}
\runtitle{BIDIFAC+ for pan-omics pan-cancer analysis}

%%%%%%%%%%%%%%%%%%%%%%%%%%%%%%%%%%%%%%%%%%%%%%%%%%%%%%%%%%%%%%%%%%%%%%%%%%%%%%

\begin{aug}
%%%%%%%%%%%%%%%%%%%%%%%%%%%%%%%%%%%%%%%%%%%%%%
%%Only one address is permitted per author. %%
%%Only division, organization and e-mail is %%
%%included in the address.                  %%
%%Additional information can be included in %%
%%the Acknowledgments section if necessary. %%
%%%%%%%%%%%%%%%%%%%%%%%%%%%%%%%%%%%%%%%%%%%%%%
\author[A]{\fnms{Eric F.} \snm{Lock}\ead[label=e1]{elock@umn.edu}},
\author[B]{\fnms{Jun Young} \snm{Park}\ead[label=e2]{junjy.park@utoronto.ca}}
\and
\author[C]{\fnms{Katherine A.} \snm{Hoadley}\ead[label=e3]{hoadley@med.unc.edu}}
%%%%%%%%%%%%%%%%%%%%%%%%%%%%%%%%%%%%%%%%%%%%%%
%% Addresses                                %%
%%%%%%%%%%%%%%%%%%%%%%%%%%%%%%%%%%%%%%%%%%%%%%
\address[A]{Division of Biostatistics, School of Public Health, University of Minnesota, 
}

\address[B]{Department of Statistical Sciences, Faculty of Arts \& Science, University of Toronto,
}

\address[C]{Department of Genetics, Computational Medicine Program, University of North Carolina,
}
\end{aug}

\begin{abstract}
Several modern applications require the integration of multiple large data matrices that have shared rows and/or columns.  %For example, a comprehensive investigation of molecular heterogeneity in cancer requires the integration of multiple omics platforms across multiple types of cancer, i.e., \emph{pan-omics pan-cancer} analysis. 
For example, cancer studies that integrate multiple omics platforms across multiple types of cancer, \emph{pan-omics pan-cancer analysis}, have extended our knowledge of molecular heterogeneity beyond what was observed in single tumor and single platform studies. However, these studies have been limited by available statistical methodology. We propose a flexible approach to the simultaneous factorization and decomposition of variation across such \emph{bidimensionally linked} matrices, BIDIFAC+.  BIDIFAC+ decomposes variation into a series of low-rank components that may be shared across any number of row sets (e.g., omics platforms) or column sets (e.g., cancer types).  This builds on a growing literature for the factorization and decomposition of linked matrices, which has primarily focused on multiple matrices that are linked in one dimension (rows or columns) only.  Our objective function extends nuclear norm penalization, is motivated by random matrix theory, gives a unique decomposition under relatively mild conditions, and can be shown to give the mode of a Bayesian posterior distribution.   We apply BIDIFAC+ to pan-omics pan-cancer data from TCGA, identifying shared and specific modes of variability across $4$ different omics platforms and $29$ different cancer types.  
\end{abstract}

\end{frontmatter}

\section{Introduction}
\label{intro}

Data collection and curation for the Cancer Genome Atlas (TCGA) program completed in 2018, providing a unique and valuable public resource for comprehensive studies of molecular profiles across several types of cancer \citep{hutter2018cancer}.  The database includes  information from several molecular platforms for over $10,000$ tumor samples from individuals representing $33$ types of cancer.  The molecular platforms capture signal at different 'omics levels (e.g., the genome, epigenome, transcriptome and proteome), which are biologically related and can each influence the behavior of the tumor.  Thus, when studying molecular signals in cancer it is often necessary to consider data from multiple omics sources at once.  This and other applications have motivated a very active research area in statistical methods for multi-omics integration.  

A common task in multi-omics applications is to jointly characterize the molecular heterogeneity of the samples.  Several multi-omics methods have been developed for this purpose, which can be broadly categorized by (1) clustering methods that identify molecularly distinct subtypes of the samples \citep{huo2017integrative, lock2013bayesian,gabasova2017clusternomics}, (2) factorization methods that identify continuous lower-dimensional patterns of molecular variability \citep{lock2013joint, argelaguet2018multi,gaynanova2017structural}, or methods that combine aspects of (1) and (2) \citep{shen2013sparse, mo2017fully,hellton2016integrative}.  These extend classical approaches, such as (1) k-means clustering and (2) principal components analysis, to the multi-omics context, allowing the exploration of heterogeneity that is shared across the different 'omics sources while accounting for their differences.

Several multi-omics analyses have been performed on the TCGA data, including flagship publications for each type of cancer (e.g., see \cite{cancer2012comprehensive,cancer2014comprehensive,verhaak2010integrated}).  These have revealed striking molecular heterogeneity within    each classical type of cancer, which is often clinically relevant.  However,  restricting an analysis to a particular type of cancer sacrifices power to detect important genomic changes that are present across more than one cancer type. In 2013 TCGA began the Pan-Cancer Analysis Project, motivated by the observation that ``cancers of disparate organs reveal many shared features, and, conversely, cancers from the same organ are often quite distinct" \citep{weinstein2013cancer}.  Subsequently, several pan-cancer studies have identified important shared molecular alterations for somatic mutations \citep{kandoth2013mutational}, copy number \citep{zack2013pan}, mRNA \citep{hoadley2014multiplatform}, and protein abundance \citep{akbani2014pan}.  However, a multi-omics analysis found that pan-cancer molecular heterogeneity is largely dominated by cell-of-origin and other factors that define the classical cancer types \citep{hoadley2018cell}.

In this study we do not focus on baseline molecular differences between the cancer types.  Rather, we focus on whether patterns of variability within each cancer type are shared across cancer types, i.e., whether multi-omic molecular profiles that drive heterogeneity in one type of cancer also drive heterogeneity in other cancers.  Systematic investigations of heterogeneity in a pan-omics and pan-cancer context are presently limited by a lack of principled and computationally feasible statistical approaches for the comprehensive analysis of such data.  In particular, the data take the form of \emph{bidimensionally linked matrices}, i.e., multiple large matrices that may share row sets (here, defined by the omics platforms) or column sets (here, defined by the cancer types); this is illustrated in Figure~\ref{fig:bididiag} and the formal framework is described in Section~\ref{framework}.  Such bidimensional integration problems are increasingly encountered in practice, particularly for biomedical applications that involve multiple omics platforms and multiple sample cohorts that may correspond to different studies, demographic strata, species, diseases or disease subtypes.      

In this article we propose a flexible approach to the simultaneous factorization and decomposition of variation across bidimensionally linked matrices, BIDIFAC+.  Our approach builds on a growing literature for the factorization and decomposition of linked matrices, which we review in Section~\ref{existing}.  However, previous methods have focused on multiple matrices that are linked in just one dimension (rows or columns), or assume that shared signals must be present across all row sets or column sets.  This is limiting for pan-cancer analysis and other applications, where we expect patterns of variation that are shared across some, but not necessarily all, cancer types and omics platforms.   With this motivation, our proposed approach decomposes variation into a series of low-rank components that may be shared across any number of row sets (e.g., omics platforms) or column sets (e.g., cancer types).  We develop a new approach to model selection, and new estimating algorithms, to accommodate this more flexible decomposition. We establish theoretical results, most notably concerning the uniqueness of the decomposition without orthogonality constraints, which are entirely new for linked matrix decompositions.  Moreover, we show how BIDIFAC+ can improve the estimation of underlying structures over existing methods even in the more familiar context for which matrices are linked in just one dimension.        %We

\begin{figure}[!h]
\includegraphics[scale=0.58]{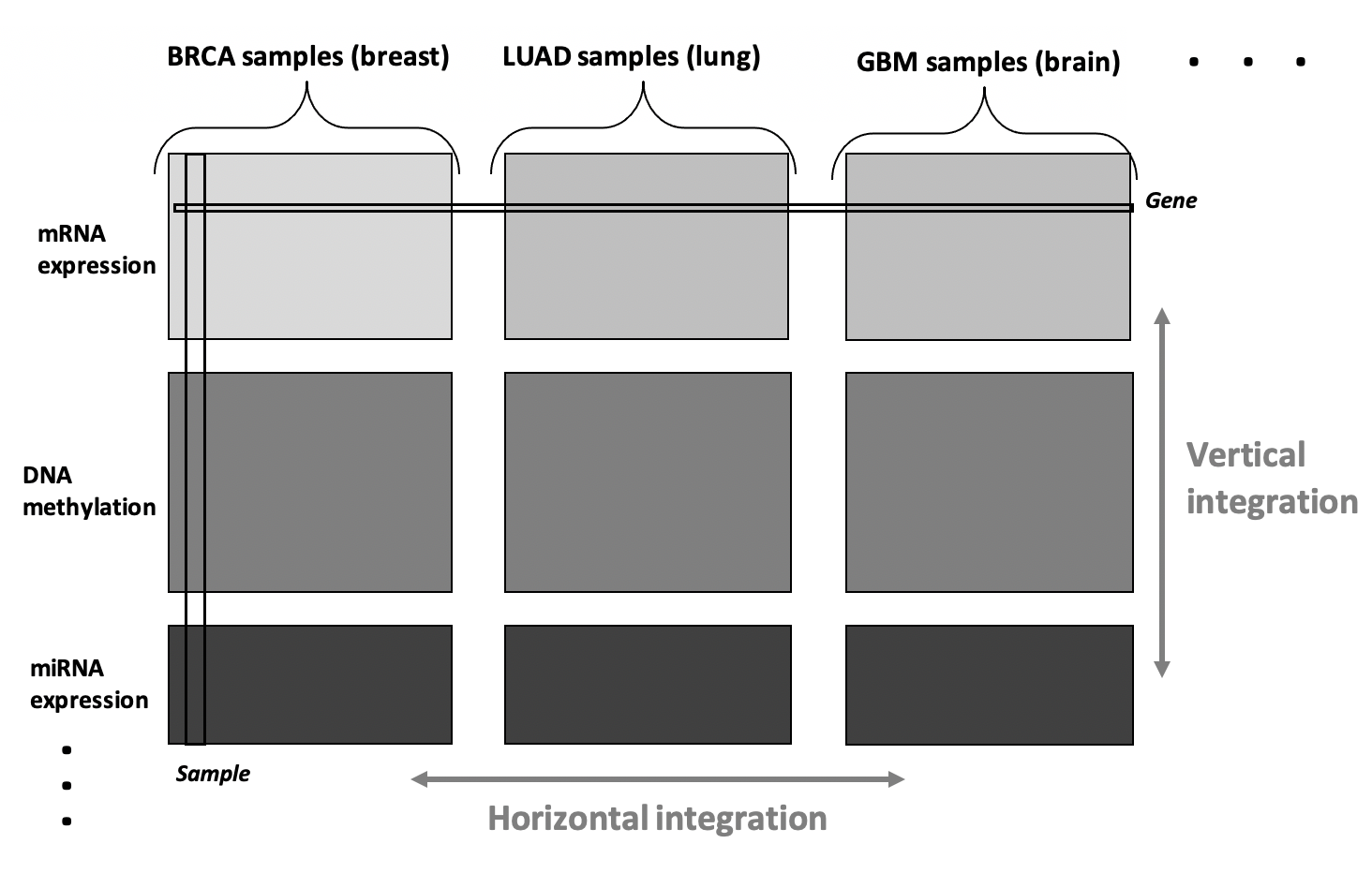}
\label{fig:bididiag}
\caption{Bidimensional integration of pan-omics pan-cancer data.}
\end{figure}

\section{Formal framework and notation}
\label{framework}

Here we introduce our framework and notation for pan-omics pan-cancer data.  Let $\X_{ij}: M_i \times N_j$ denote the data matrix for omics data source $i$ and sample set (i.e., cancer type) $j$ for $j=1,\hdots,J$ and $i=1,\hdots,I$. Columns of $\X_{ij}$ represent samples, and rows represent variables (e.g., genes, miRNAs, proteins).  The sample sets of size $\N=[N_1,\hdots,N_J]$ are consistent across each omics source, and the features measured for each omics source $\M = [M_1,\hdots,M_I]$ are consistent across sample sets. 
As illustrated in Figure~\ref{fig:bididiag}, the collection of available data can be represented as a single data matrix $\X_{\bigcdot \bigcdot}:M \times N$ where $M= M_1+\hdots+M_I$ and $N = N_1+\hdots+N_J$, by horizontally and vertically concatenating its constituent blocks:   
\begin{align} \X_{\bigcdot \bigcdot} = \left [ \begin{array}{c c c c} \X_{11} & \X_{12} & \hdots & \X_{1J}  \\  \vdots & \vdots & \ddots & \vdots \\ \X_{I1} & \X_{I2} & \hdots & \X_{IJ} \end{array} \right ] \; \text{  where  } \X_{ij} \text{  are  } M_i \times N_j. \label{bimat}\end{align}
Similarly, $\X_{i \bigcdot}$ defines the concatenation of omics source $i$ across cancer types and $\X_{\bigcdot j}$ defines the concatenation of cancer type $j$ across omics sources:
\begin{align*} \X_{i \bigcdot} = \left[\X_{i1} \; \hdots \; \X_{iJ} \right]  \text{  ,  } \; \X_{\bigcdot j} = \left[ \X_{1j}'  \; \hdots \;  \X_{Ij}' \right]' . \end{align*}
The notation $\X_{ij}[\bigcdot,n]$ defines the n'th column of matrix $ij$, $\X_{ij}[m,\bigcdot]$ defines the m'th row, and $\X_{ij}[m,n]$ defines the entry in row $m$ and column $n$.  In our context, the entries are all quantitative, continuous measurements; missing data are addressed in Section~\ref{missing}.

We will investigate shared or unique patterns of systematic variability (i.e., heterogeneity) among the constituent data blocks. We are not interested in baseline differences between the different omics platforms or sample sets,  and so after routine preprocessing the data will be centered so that the mean of the entries within each data block, $\X_{ij}$, is $0$.  Moreover, to resolve the disparate scale of the data blocks, each block will be scaled to have comparable variability as described in Section~\ref{scaling}.

In what follows, $||\A||_F$ denotes the Frobenius norm for any given matrix, so that $||\A||_F^2$ is the sum of squared entries in $\A$.  The operator $||\A||_*$ denotes the nuclear norm of $\A$, which is given by the sum of the singular values in $\A$; that is, if $\A: M \times N$ has ordered singular values $\D[1,1], \D[2,2], \hdots$, then $||\A||_* = \sum_{r=1}^{\mbox{min}(M,N)} \D[r,r]$.

\section{Existing integrative factorization methods}
\label{existing}

There is now an extensive literature on the integrative factorization and decomposition of multiple linked datasets that share a common dimension. Much of this methodology is motivated by multi-omics integration, i.e., \emph{vertical integration} of multiple matrices $\{\X_{11}, \X_{21},\hdots, \X_{I1}\}$ with shared columns in the setting of Section~\ref{framework}.  For example, the Joint and Individual Variation Explained (JIVE) method \citep{lock2013joint, oconnell2016} decomposes variation into \emph{joint} components that are shared among multiple omics platforms and \emph{individual} components that only explain substantial variability in one platform.  This distinction simplifies interpretation, and also improves accuracy in recovering underlying signals. Accuracy improves because structured individual variation can interfere with finding important joint signal, just as joint structure can obscure important signal that is individual to a data source.  The factorized JIVE decomposition is 
\begin{align} \label{modelPCA} \X_{i1} = \U_{i} \V^T+\U^*_{i} \V_{i}^T + \E_{i} \; \; \text{ for } i=1,\hdots,I. \end{align} 
Joint structure is represented by the common score matrix $\V: N_1 \times R$,  which summarize patterns in the samples that explain variability across multiple omics platforms. The loading matrices $\U_{i}: M_i \times R$ indicate how these joint scores are expressed in the rows (variables) of platform $i$.  The score matrices $\V_{i}: N_1 \times R_i$ summarize sample patterns specific to platform $i$, with loadings $\U^*_{i}$.   Model~\eqref{modelPCA} can be equivalently represented as a sum of low-rank matrices
\begin{align} \label{modelRank} \X_{\bigcdot 1} = \S_{\bigcdot}^{(0)}+\sum_{i=1}^I \S_{\bigcdot}^{(i)} + \E_{\bigcdot} \end{align} 
where $\S_{\bigcdot}^{(0)} = \U_{\bigcdot} \V^T$ is of rank $R$ and $\S_{\bigcdot}^{(i)}=[\S_{1}^{(i)'} \; \hdots \; \S_{I}^{(i)'}]'$ is the matrix of rank $R_i$ given by the individual structure for platform $i$ and zeros elsewhere:
\[ \S_{i'}^{(i)} = \begin{cases}
 	\0_{M_{i'} \times N} \; \text{if } \; i' \neq i\\
 	\U_{i'}^{*} \V_i^T \; \text{ if }\;  i' = i.
 \end{cases}\]
%\[\S_{i} = \left [ \begin{array}{c} \0_{M_1 \times N} \\ \vdots \\ \U^*_{i} \V_{i} \\ \0_{M_{i+1} \times N} \\ \vdots \\ \0_{M_{I} \times N} \end{array} \right].    \]
Several other methods result in a factorized decomposition similar to that in \eqref{modelPCA} and \eqref{modelRank}, including approaches that allow for different distributional assumptions on the constituent matrices \citep{li2018general, zhu2018generalized}, non-negative factorization \citep{yang2015non}, and the incorporation of covariates \citep{li2017incorporating}.  The Structural Learning and Integrative Decomposition (SLIDE) method \citep{gaynanova2017structural} allows for a more flexible decomposition in which some components are only partially shared across a subset of the constituent data matrices.  SLIDE extends model~\eqref{modelRank} to the more general decomposition 
\begin{align} \label{slide} \X_{\bigcdot 1} = \sum_{k=1}^K \S_{\bigcdot}^{(k)} + \E_{\bigcdot} \end{align} 
where $\S_{\bigcdot}^{(k)} = [\S_{1}^{(k)'} \; \hdots \; \S_{I}^{(k)'}]'$ is a low-rank matrix with non-zero values for some subset of the sources that is identified by a binary matrix $\R: I \times K$
and \[ \S_i^{(k)} = \begin{cases}
 	\0_{M_i \times N} & \text{ if } \R[i,k]=0 \\ 
 	\U_{i}^{(k)} \V^{(k) T} & \text{ if } \R[i,k]=1
 \end{cases}.\]
 Here, $\V^{(k)}$ gives scores that explain variability for only those patterns for the omics sources identified by $\R[\bigcdot,k]$.  
 
 The BIDIFAC approach \citep{park2019integrative} is designed for the decomposition of bidimensionally linked matrices as in~\eqref{bimat}.  Its low-rank factorization can be viewed as an extension of that for JIVE, decomposing variation into structure that is shared globally ($\text{G}$), across rows ($\text{Row}$), across columns ($\text{Col}$), or individual to the constituent matrices ($\text{Ind}$).  Following \eqref{modelRank} for JIVE and \eqref{slide} for SLIDE, its full decomposition can be expressed as
        \begin{align} \label{modelbidi} \X_{\bigcdot \bigcdot} = \S_{\bigcdot \bigcdot}^{(\text{G})}+\sum_{i=1}^I \S_{\bigcdot \bigcdot}^{(i, \text{Row})} + \sum_{j=1}^J \S_{\bigcdot \bigcdot}^{(j, \text{Col})}+\sum_{i=1}^I \sum_{j=1}^J \S_{\bigcdot \bigcdot}^{(i, j, \text{Ind})}+\E_{\bigcdot \bigcdot} \end{align}
        where $\S_{i j}^{(\text{G})}=\U_{i}^{(\text{G})} \V_{j}^{(\text{G}) T}$,\\  $\S_{i' j'}^{(i,\text{Row})} = \begin{cases}
 	\0_{M_i \times N_j} \; &\text{ if } i' \neq i \\
 	\U_{i}^{(i,\text{Row})} \V_{j}^{(i,\text{Row}) T} \; &\text{ if } i' = i
 \end{cases}$, $\; \; \S_{i' j'}^{(j,\text{Col})} = \begin{cases}
 	\0_{M_i \times N_j} \; &\text{ if } j' \neq j \\
 	\U_{i}^{(j,\text{Col})} \V_{j}^{(j,\text{Col}) T} \; &\text{ if } j' = j
 \end{cases}$,\\ and       
       \[ \S_{i' j'}^{(i,j,\text{Ind})} = \begin{cases}
 	\0_{M_i \times N_j} \; &\text{ if } i' \neq i \text{ or } j' \neq j  \\
 	\U_{i}^{(i,j,\text{Ind})} \V_{j}^{(i,j,\text{Ind}) T} \; &\text{ if } i' = i \text{ and } j' = j
 \end{cases}. \]

\section{Proposed model}
\label{model}
We consider a flexible factorization  of bidimensionally linked data that combines aspects of the BIDIFAC and SLIDE models.  Our full decomposition can be expressed as 
  \begin{align} \label{bidiflex} \X_{\bigcdot \bigcdot} = \sum_{k=1}^K \Sk + \E_{\bigcdot \bigcdot}, \end{align}  
  where \begin{align*} \S^{(k)}_{\bigcdot \bigcdot} = \left [ \begin{array}{c c c c} \S^{(k)}_{11} & \S^{(k)}_{12} & \hdots & \S^{(k)}_{1J}  \\  \vdots & \vdots & \ddots & \vdots \\ \S^{(k)}_{I1} & \S^{(k)}_{I2} & \hdots & \S^{(k)}_{IJ} \end{array} \right ] \end{align*}
  and the presence of each $\S^{(k)}_{ij}$  is determined by a binary matrix of row indicators $\R: I \times K$ and column indicators $\C: J \times K$:
  \[\S^{(k)}_{ij} = \begin{cases}
 	\0_{M_i \times N_j} \; & \text{ if } \R[i,k] = 0 \text{ or } \C[j,k] = 0\\
 	\U_{i}^{(k)} \V_{j}^{(k) T} & \text{ if } \R[i,k] = 1 \text{ and } \C[j,k] = 1 \end{cases}. \] 
 	Each $\S^{(k)}_{\bigcdot \bigcdot}$ gives a low-rank \emph{module} that explains variability within the omics platforms identified by $\R[\bigcdot,k]$ and the cancer types identified by $\C[\bigcdot,k]$.  By requiring $\R[i,k] = 1$ \emph{and} $\C[j,k] = 1$, the module is non-zero on a contiguous submatrix.   There are in total $(2^I-1)(2^J-1)$ such submatrices, so by default we set $K=(2^I-1)(2^J-1)$ and let $\R$ and $\C$ enumerate all possible modules (see Appendix A).   The SLIDE decomposition \eqref{slide} is a special case when $J=1$ or $I=1$ (i.e., unidimensional integration); the BIDIFAC model \eqref{modelbidi} is a special case where each column of $\R$ and $\C$ contains either entirely $1$'s (i.e., all rows or columns included) or just one $1$ (i.e., just one row set or column set included).   In practice, if the row and column set for a structural module is not included, it may be subsumed into a larger module or broken into separate smaller modules.  The matrix $\E_{\bigcdot \bigcdot}$ is an error matrix, whose entries are assumed to be sub-Gaussian with mean $0$ and variance 1 after scaling (see Section~\ref{scaling}.   

Let the rank of each module be rank$(\S^{(k)}_{\bigcdot \bigcdot})=R_k$, so that the dimensions of the non-zero components of the factorization are $\U_{i}^{(k)}: M_i \times R_k$ and $\V_{j}^{(k)}: N_j \times R_k$.  The $r$'th component of the $k$'th module gives a (potentially multi-omic) molecular profile $\{\U_{i}^{(k)}[\bigcdot,r] : \R[i,k]=1\}$ that explains variability within those cancer types defined by $\C[\cdot,k]$ with corresponding sample scores $\{\V_{j}^{(k)}[r,\bigcdot] : \C[j,k]=1\}$.     

\section{Objective function}
\label{objective}
To estimate model~\eqref{bidiflex}, we minimize a least squares criterion with a structured nuclear norm penalty: 
% \begin{align} \underset{\R,\C,\Sol}{\mbox{argmin}} \; \; ||\X_{\bigcdot \bigcdot}-\sum_{k=1}^K \Sk||^2_F + \sum_{k=1}^K 2\left(\sqrt{\M \cdot \R[\bigcdot,k]}+\sqrt{\N \cdot \C[\bigcdot,k]} \right) ||\Sk||_*  \label{obj_eq} \end{align}
\begin{align} \underset{\Sol}{\mbox{argmin}} \; \; \frac{1}{2}||\X_{\bigcdot \bigcdot}-\sum_{k=1}^K \Sk||^2_F + \sum_{k=1}^K \lambda_k ||\Sk||_*  \label{obj_eq} \end{align}
subject to $\S^{(k)}_{ij} = \0_{M_i \times N_j}$ if  $\R[i,k] = 0$ or $\C[j,k] = 0$.   
The choice of the penalty parameters $\{\lambda_k\}_{k=1}^K$ is critical, and must satisfy the conditions of Proposition~\ref{prop_nec} to allow for non-zero estimation of each module. 
\begin{prop}
\label{prop_nec}
Under objective~\eqref{obj_eq}, the following are necessary to allow for each $\hatSk$ to be non-zero 
\begin{enumerate}
\item If for $k'\neq k$ the rows and columns of module $k'$ are contained within those for module $k$, $\bR[i, k]-\bR[i, k'] \geq 0$ $\forall$ $i$ and $\bC[j, k]-\bC[j, k'] \geq 0$ $\forall$ $j$, then $\lambda_k>\lambda_{k'}$.
\item If $\mathcal{I}_k \subset \{1,\hdots,k-1,k+1,\hdots,K\}$ is any subset of modules that together cover the rows and columns of module $k$, $\sum_{j\in \mathcal{I}_k} \bR[\bigcdot,j]= r \cdot \bR[\bigcdot,k]$ and $\sum_{j\in \mathcal{I}_k} \bC[\bigcdot,j]= c \cdot \bC[\bigcdot,k]$ for positive integers $r$ and $c$, then $\lambda_k< \sum_{j\in \mathcal{I}_k} \lambda_{j}$.
\end{enumerate}
\end{prop}
We determine the $\lambda_k$'s by random matrix theory, motivated by two well-known results for a single matrix that we repeat here in Propositions~\ref{prop1} and~\ref{prop2}.
\begin{prop}
\label{prop1} 
\citep{mazumder2010spectral} Let $\U \D \V^T$ be the SVD of a matrix $\X$. The approximation $\A$ that minimizes 
\begin{align} \frac{1}{2}||\X-\A||^2_F + \lambda ||\A||_* \label{lowrank} \end{align}
is $\A=\U\tilde{\D}\V^T$, where $\tilde{\D}$ is diagonal with entries $\tilde{\D}[r,r] = \mbox{max}(\D[r,r]-\lambda,0)$. 
\end{prop}

\begin{prop}
\label{prop2}
\citep{rudelson2010non} Let $D[1,1]$ be the largest singular value of a matrix $\E: M \times N$ of independent entries with mean $0$, variance $\sigma^2$, and finite fourth moment. Then, $\D[1,1] \approx  \sigma(\sqrt{M}+\sqrt{N})$ as $M,N \rightarrow \infty$, and if the entries of $\E$ are Gaussian $\mathbb{E}(\D[1,1]) \leq \sigma(\sqrt{M}+\sqrt{N})$ for any $M, N$.   
\end{prop}
Fixing $\lambda=\sigma(\sqrt{M}+\sqrt{N})$ is a reasonable choice for the matrix approximation task in~\eqref{lowrank}, because it keeps only those components $r$ whose signal is greater than that expected for independent error by Proposition~\ref{prop2}: $\D[r,r]>\sigma(\sqrt{M}+\sqrt{N})$ \citep{shabalin2013reconstruction}.  For our context $\sigma=1$ after normalizing as discussed in Section~\ref{scaling}, and thus we set $ \lambda_k=\sqrt{\R[\bigcdot,k]\cdot \M}+\sqrt{\C[\bigcdot,k]\cdot \N}$, where $\R[\bigcdot,k]\cdot \M \times \C[\bigcdot,k]\cdot \N$ gives the dimensions of the non-zero sub-matrix for $\Sk$: 
\[\R[\bigcdot,k] \cdot \M = \sum_{i=1}^I M_i \R[i,k] \; \; \text{and} \; \; \C[\bigcdot,k] \cdot \N = \sum_{j=1}^J N_j \C[j,k]. \] 
Our choice of $\lambda_k$ is motivated to distinguish signal from noise in module $\Sk$, conditional on the other modules.  Moreover, it is guaranteed to satisfy the necessary conditions in Proposition~\ref{prop_nec}, which we establish in Proposition~\ref{prop_sat}.
\begin{prop}
\label{prop_sat}
Setting $\lambda_k=\sqrt{\R[\bigcdot,k]\cdot \M}+\sqrt{\C[\bigcdot,k]\cdot \N}$ in~\eqref{obj_eq} satisfies the necessary conditions of Proposition~\ref{prop_nec}.
\end{prop}
A similarly motivated choice of penalty weights is used in the BIDIFAC method, which solves an equivalent objective under the restricted scenario where the columns of $\R$ and $\C$ are fixed and contain either entirely $1$'s (i.e., all rows or columns included) or just one $1$ (i.e., just one row set or column set included). Thus, we call our more flexible approach BIDIFAC+.  

It is often infeasible to explicitly consider each of the $K=(2^I-1)(2^J-1)$ possible modules in \eqref{obj_eq}, and the solution is often sparse,  with $\hatSk=\0$ for several $k$.  Thus,  in practice we also optimize over the row and column sets $\R$ and $\C$ for some smaller number of modules $\tilde{K} \ll K$:
 \begin{align} \underset{\R,\C,\{\S_{\bigcdot \bigcdot}^{(k)}\}_{k=1}^{\tilde{K}}}{\mbox{argmin}} \; \; \frac{1}{2}||\X_{\bigcdot \bigcdot}-\sum_{k=1}^{\tilde{K}} \Sk||^2_F + \sum_{k=1}^{\tilde{K}} \left(\sqrt{\M \cdot \R[\bigcdot,k]}+\sqrt{\N \cdot \C[\bigcdot,k]} \right) ||\Sk||_*  \label{obj_eq_rc}. \end{align}
Note that if $\tilde{K}$ is greater than the number of non-zero modules, then the solution to \eqref{obj_eq_rc} is equivalent to the solution to \eqref{obj_eq} in which $\R$ and $\C$ are fixed and enumerate all possible modules.  If $\tilde{K}$ is not greater than the number of non-zero modules, then the solution to \eqref{obj_eq_rc} can still be informative as the set of $\tilde{K}$ modules that together give the best structural approximation via \eqref{obj_eq}. In this case, it helps to order the estimated modules by variance explained;  if the $\tilde{K}$'th module still explains substantial variance, consider increasing $\tilde{K}$.  Moreover, we suggest assessing the sensitivity of results to different values of $\tilde{K}$.

\section{Estimation}

\subsection{Scaling}
\label{scaling}
We center each dataset $\X_{ij}$ to have mean $0$, and scale each dataset to have residual variance var$(\E_{ij}$) of approximately $1$.  Such scaling requires an estimate of the residual variance for each dataset. By default we use the median absolute deviation estimator $\hat{\sigma}^2_{MAD}$ of \cite{gavish2017optimal}, which is motivated by random-matrix theory under the assumption that $\X_{ij}$ is composed of low-rank structure and mean $0$ independent noise of variance $\sigma^2$.  We estimate $\hat{\sigma}^2_{MAD}$ for the unscaled data $\X_{ij}^{\text{unscaled}}$, and set $\X_{ij} = \X_{ij}^{\text{unscaled}}/\hat{\sigma}_{MAD}$.   An alternative approach is to scale each dataset to have overall variance $1$, var$(\X_{ij})$=1, which is more conservative because var$(\E_{\bigcdot \bigcdot}) \leq \mbox{var}(\X_{ij})$; thus, this approach results in relatively larger $\lambda_k$ in the objective function and leads to sparser overall ranks. Yet another approach is to normalize each data block to have unit Frobenius norm, as in \citet{lock2013joint}.  However, our default choice of penalty parameters $\lambda_k$ in Section~\ref{objective} is theoretically motivated by the assumption that residual variances are the same across the different data blocks. 
 
\subsection{Optimization algorithm: fixed modules}
\label{est.fixed}

We estimate across all modules $k=1,\hdots,K$  simultaneously by iteratively optimizing the objectives in Section~\ref{objective}.   First assume the row and column inclusions for each module, defined by $\R$ and $\C$, are fixed as in objective~\eqref{obj_eq}.  To estimate $\Sk$ given the other modules $\{\S_{\bigcdot \bigcdot}^{(k')}\}_{k' \neq k}$, we can apply the soft-singular value estimator in Proposition~\ref{prop1} to the residuals on the submatrix defined by $\R[\bigcdot,k]$ and $\C[\bigcdot,k]$.
  The iterative estimation algorithm proceeds as follows:
\begin{enumerate}
\item Initialize $\hat{\S}_{\bigcdot \bigcdot}^{(k)} = \0_{M \times N}$ for $k=1,\hdots,K$.	
\item For $k=1,\hdots,K$: 
\begin{enumerate}
\item Compute the residual matrix $\X_{\bigcdot \bigcdot}^{(k)} = \X_{\bigcdot \bigcdot} - \sum_{k'\neq k} \hat{\S}_{\bigcdot \bigcdot}^{(k')}$
\item Set $\X_{i j}^{(k)} = \0_{M_i \times N_j}$ where $\R[i,k]=0$ or $\C[j,k]=0$
\item Compute the SVD of $\X_{\bigcdot \bigcdot}^{(k)}$, $\X_{\bigcdot \bigcdot}^{(k)} = \U_{\bigcdot}^{(k)} \D^{(k)} \V_{\bigcdot}^{(k)}$
\item Update $\hat{\S}_{\bigcdot \bigcdot}^{(k)} = \U_{\bigcdot}^{(k)} \hat{\D}^{(k)} \V_{\bigcdot}^{(k)}$ where $\hat{\D}[r,r]=\mbox{max}(\D[r,r]-\lambda_k,0)$ for $r=1,2,\hdots$ .
\end{enumerate}
\item Repeat step 2.\ until convergence of the objective function~\eqref{obj_eq_rc}
%\begin{align}  ||\X_{\bigcdot \bigcdot}-\sum_{k=1}^K \hatSk||^2_F + \sum_{k=1}^K 2\left(\sqrt{\M \cdot \R[\bigcdot,k]}+\sqrt{\N \cdot \C[\bigcdot,k]} \right) ||\Sk||_*.  \label{obj_fun} \end{align}
\end{enumerate}
Step 2(d) minimizes the objective \eqref{obj_eq} for $\hatSk$ given $\{\S_{\bigcdot \bigcdot}^{(k')}\}_{k' \neq k}$, by Proposition~\ref{prop1}.  In this way, we iteratively optimize \eqref{obj_eq} over the $K$ modules $\{\S_{\bigcdot \bigcdot}^{(k)}\}_{k=1}^K$ until convergence.  

\subsection{Optimization algorithm: dynamic modules}
\label{est.dynamic}
If the row and column inclusions $\R$ and $\C$ are not predetermined, we incorporate additional sub-steps to estimate the non-zero submatrix defined by $\R[\bigcdot,k]$ and $\C[\bigcdot,k]$ for each module to optimize~\eqref{obj_eq_rc}.  We use a dual forward-selection procedure to iteratively determine the optimal row-set $\R[\bigcdot,k]$ with columns $\C[\bigcdot,k]$ fixed, and the column-set $\C[\bigcdot,k]$ with rows $\R[\bigcdot,k]$ fixed, until convergence prior to estimating each $\S_{\bigcdot \bigcdot}^{(k)}$. The iterative estimation algorithm proceeds as follows:
\begin{enumerate}
\item Initialize $\hat{\S}_{\bigcdot \bigcdot}^{(k)} = \0_{M \times N}$ for $k=1,\hdots,K$. \item Initialize $\hat{\C}[j,k]=1$ for $j=1,\hdots,J$.	
\item For $k=1,\hdots,K$: 
\begin{enumerate}
\item Compute the residual matrix $\X_{\bigcdot \bigcdot}^{(k)} = \X_{\bigcdot \bigcdot} - \sum_{k'\neq k} \hat{\S}_{\bigcdot \bigcdot}^{(k')}$
\item Update $\hat{\R}[\bigcdot,k]$ and $\hat{\C}[\bigcdot,k]$ as follows:
\begin{enumerate}
\item With $\hat{\C}[\bigcdot,k]$ fixed, update $\hat{\R}[\bigcdot,k]$ by forward selection, beginning with $\hat{\R}[\bigcdot,k]=\0$ and iteratively adding rows $i$ ($\hat{\R}[i,k]=1$) to minimize the objective~\eqref{obj_eq_rc}.  
\item With $\hat{\R}[\bigcdot,k]$ fixed, update $\hat{\C}[\bigcdot,k]$ by forward selection, beginning with $\hat{\C}[\bigcdot,k]=\0$ and iteratively adding columns $j$ ($\hat{\C}[i,k]=1$) to minimize the objective~\eqref{obj_eq_rc}.
\item Repeat steps i.\ and ii.\ until convergence of the chosen row and column sets $\hat{\C}[\bigcdot,k]$ and $\hat{\R}[\bigcdot,k]$.        
\end{enumerate}
\item Set $\X_{i j}^{(k)} = \0_{M_i \times N_j}$ where $\hat{\R}[i,k]=0$ or $\hat{\C}[j,k]=0$
\item Compute the SVD of $\X_{\bigcdot \bigcdot}^{(k)}$, $\X_{\bigcdot \bigcdot}^{(k)} = \U_{\bigcdot}^{(k)} \D^{(k)} \V_{\bigcdot}^{(k)}$
\item Update $\hat{\S}_{\bigcdot \bigcdot}^{(k)} = \U_{\bigcdot}^{(k)} \hat{\D}^{(k)} \V_{\bigcdot}^{(k)}$ where $\hat{\D}[r,r]=\mbox{max}(\D[r,r]-\lambda_k,0)$ for $r=1,2,\hdots$ .
\end{enumerate}
\item Repeat step 3.\ until convergence of the objective function.
\end{enumerate}
For the steps in 3b), objective ~(\ref{obj_eq_rc}) can be efficiently computed using the singular values for the residual submatrix given by $\hat{\R}[\bigcdot,k]$ and $\hat{\C}[\bigcdot,k]$, without re-estimating $\hat{\S}_{\bigcdot \bigcdot}^{(k)}$ completely at each substep. 
%Further details and pseudocode for the algorithm are provided in Appendix~\ref{algorithm.dets}.

 \subsection{Convergence and tempered regularization}
 \label{temp_reg}
The algorithms in Sections~\ref{est.fixed} and~\ref{est.dynamic} both monotonically decrease the objective function at each sub-step, and thus both converge to a coordinate-wise optimum.  For~(\ref{obj_eq}), the objective function and solution space are both convex, and the algorithm with fixed modules tends to converge to a global optimum, as observed for the BIDIFAC method \cite{park2019integrative}.  However, the stepwise updating of $\R$ and $\C$ in Section~\ref{est.dynamic} can get stuck at coordinate-wise optima, analogous to stepwise variable selection in a predictive model.   In practice, we find that the convergence of the algorithm improves substantially if the initial iterations use a high nuclear norm penalty that gradually decreases to the desired level of penalization.  Thus, in our implementation for the first iteration the penalties are set to $\tilde{\lambda}_k = \alpha \lambda_k$ for $k=1,\hdots,K$ and some $\alpha>1$. The penalties then gradually decrease over each subsequent iteration of the algorithm, before reaching the desired level of regularization ($\alpha=1$). 

\section{Uniqueness}
\label{ident}

Here we consider the uniqueness of the decomposition in~\eqref{model} under the objective~\eqref{obj_eq}.   To account for permutation invariance of the $K$ modules, throughout this section we assume that $\R$ and $\C$ are fixed and that they enumerate all of the $K=(2^I-1)(2^J-1)$ possible modules.  Note that the solution may still be sparse, with $\Sk=\0$ for some or most of the modules. Without loss of generality, we fix $\R$ and $\C$ as in Supplemental Appendix A. Then, let $\SolSpace$ be the set of possible decompositions that yield a given approximation $\hat{\X}_{\bigcdot \bigcdot}$:
$$\SolSpace = \left \{ \Sol \mid \hat{\X}_{\bigcdot \bigcdot} = \sum_{k=1}^K \Sk \right \}.$$  
If either $I>1$ or $J>1$ then the cardinality of $\SolSpace$ is infinite, i.e., there are an infinite number of ways to decompose $\hat{\X}_{\bigcdot \bigcdot}$.  Thus, model~\eqref{model} is clearly not identifiable in general, even in the no-noise case $\E_{\bigcdot \bigcdot} = \0$.  However, optimizing the structured nuclear norm penalty in \eqref{obj_eq} may uniquely identify the decomposition;  let $\fpen (\cdot)$ give this penalty:
\[\fpen (\Sol) = \sum_{k=1}^K \left(\sqrt{\R[\bigcdot,k] \cdot \M}+\sqrt{\C[\bigcdot,k] \cdot \N} \right) ||\Sk||_*.\] 
Proposition~\ref{propgen} gives an equivalence of the left and right singular vectors for any two decompositions that minimize $\fpen (\cdot)$.     
\begin{prop}
\label{propgen}
Take two decompositions  $\SolHat \in \SolSpace$ and $\SolTil \in \SolSpace$, and assume that both minimize the structured nuclear norm penalty: 
\[\fpen (\SolHat)=\fpen \left(\SolTil\right)=\underset{\SolSpace}{\min} \; \fpen  (\Sol).\] Then, $\hatSk = \U_{\bigcdot}^{(k)} \hat{\D} \V_{\bigcdot}^{(k) T}$ and $\hatSk = \U_{\bigcdot}^{(k)} \tilde{\D}^{(k)} \V_{\bigcdot}^{(k) T}$ where 
$\U_{\bigcdot}^{(k)}: M \times R_k$ and $\V_{\bigcdot}^{(k)}: N \times R_k$ have orthonormal columns, and
$\hat{\D}^{(k)}$ and $\tilde{\D}^{(k)}$ are diagonal. %, and 
%\item if $\hat{\D}^{(k)}[r,r] \neq \tilde{\D}^{(k)}[r,r]$, then one of $\hat{\D}^{(k)}[r,r]$ or $\tilde{\D}^{(k)}[r,r]$ is 0.   
\end{prop}
The proof of Proposition~\ref{propgen} uses two novel lemmas (see Supplemental Appendix B): one establishing that $\hatSk$ and $\tilSk$ must be additive in the nuclear norm, $||\hatSk+\tilSk||_* = ||\hatSk||_*+||\tilSk||_*$, and a general result establishing that any two matrices that are additive in the nuclear norm must have the equivalence in Proposition~\ref{propgen}.  

Proposition~\ref{propgen} implies that left or right singular vectors of $\hatSk$ ($\hat{\D}^{(k)}[r,r] > 0$) are either shared with $\tilSk$ (if $\tilde{\D}^{(k)}[r,r] > 0$) or orthogonal to $\tilSk$ (if $\tilde{\D}^{(k)}[r,r] = 0$).  For uniqueness, one must establish that $\hat{\D}^{(k)}[r,r] = \tilde{\D}^{(k)}[r,r]$ for all $k$ and $r$. Theorem~\ref{identTheor} gives sufficient conditions for uniqueness of the decomposition. 

%\begin{definition} 
%\label{def1}
%Let $\SolSpace$ be the set of possible decompositions of $\hat{\X}_{\bigcdot \bigcdot}$:
%$$\SolSpace = \lbrace \Sol \mid \hat{\X}_{\bigcdot \bigcdot} = \sum_{k=1}^K \Sk\rbrace.$$
%\end{definition}

%\begin{definition}
%Let $\fpen (\cdot)$ give the structured nuclear norm penalty from objective~\eqref{obj_eq}: 
%\[\fpen (\Sol) = \sum_{k=1}^K 2\left(\sqrt{\R[\bigcdot,k] \cdot \M}+\sqrt{\C[\bigcdot,k] \cdot \N} \right) ||\Sk||_*.\]
%\end{definition}

\begin{theorem}
\label{identTheor}  Consider $\SolHat \in \SolSpace$ and let $\U_{\bigcdot}^{(k)} \hat{\D}^{(k)} \V_{\bigcdot}^{(k) T}$ give the SVD of $\hatSk$ for $k=1,\hdots,K$. The following three properties uniquely identify $\SolHat$.
\begin{enumerate}
\item $\SolHat$ minimizes $\fpen (\cdot)$ over $\SolSpace$,   
\item $\{\hat{\U}_i^{(k)}[\bigcdot,r]: \R[i,k] = 1 \text{ and } \hat{\D}^{(k)}[r,r]>0\}$ are linearly independent for $i=1,\hdots I$,  
\item $\{\hat{\V}_j^{(k)}[\bigcdot,r]: \C[j,k] = 1 \text{ and } \hat{\D}^{(k)}[r,r]>0\}$ are linearly independent for $j=1,\hdots,J$.
\end{enumerate} 
\end{theorem}

The linear independence conditions (2.\ and 3.) are in general not sufficient, and several related integrative factorization methods such as JIVE \citep{lock2013joint} and SLIDE \citep{gaynanova2017structural} achieve identifiability via stronger orthogonality conditions across the terms of the decomposition.  Theorem~\ref{identTheor} implies that orthogonality is not necessary under the penalty $\fpen (\cdot)$.    Conditions 2.\ and 3.\ are straightforward to verify for any $\SolHat$, and they will generally hold whenever the ranks in the estimated factorization are small relative to the dimensions $\{M_i\}_{i=1}^I$ and $\{N_j\}_{j=1}^J$.  Moreover, the conditions of Theorem~\ref{identTheor} are only sufficient for uniqueness; there may be cases for which the minimizer of $\fpen (\cdot)$ is unique and the terms of its decomposition are not linearly independent. Theorem~\ref{identTheor} implies uniqueness of the BIDIFAC decomposition \citep{park2019integrative} under linear independence as a special case, which is a novel result.       

%Linear independence alone (conditions 2.\ and 3.\ ) is not sufficient for identifiability, as other methods make 

\section{Bayesian interpretation}
\label{bayesian}

Express the BIDIFAC+ model \eqref{bidiflex} in factorized form
\begin{align}
	\X_{\bigcdot \bigcdot} = \sum_{k=1}^K \U_{\bigcdot}^{(k)} \V_{\bigcdot}^{(k) T}+ \E_{\bigcdot \bigcdot}
\end{align}
where 
\begin{align} \U_{\bigcdot}^{(k)'}=[\U_1^{(k)'} \cdots \U_I^{(k)'}]', \text{ with } \U_i^{(k)}=M_1 \times R_k \text{ and } \U_i^{(k)}=\0_{M_1 \times R_k} \text{ if } \R[i,k]=1 \label{Ustuff} \end{align} for all $i$ and $k$, and 
\begin{align} \V_{\bigcdot}^{(k)}=[\V_1^{(k)} \cdots \V_J^{(k)}], \text{ with } \V_i^{(k)}=N_j \times R_k  \text{ and } \V_i^{(k)}=\0_{N_j \times R_k} \text{ if } \C[j,k]=1 \label{Vstuff} \end{align}
for all $j$ and $k$.
 The structured nuclear norm objective~\eqref{obj_eq} can also be represented by $L_2$ penalties on the factorization components $\U_{\bigcdot}^{(k)}$ and $\V_{\bigcdot}^{(k)}$.  We formally state this equivalence in Proposition~\ref{prop3}, which extends analogous results for a single matrix \citep{mazumder2010spectral} and for the BIDIFAC framework \citep{park2019integrative}.
\begin{prop}
\label{prop3}
Fix $\R$ and $\C$. Let $\{\hat{\U}_{\bigcdot}^{(k)}\}_{k=1}^K$ and  $\{\hat{\V}_{\bigcdot}^{(k)}\}_{k=1}^K$ minimize 
\begin{align}
 ||\X_{\bigcdot \bigcdot}-\sum_{k=1}^K \U_{\bigcdot}^{(k)},\V_{\bigcdot}^{(k)}||^2_F + \sum_{k=1}^K \lambda_k \left(||\U_{\bigcdot}^{(k)}||_F^2 + ||\V_{\bigcdot}^{(k)}||_F^2 \right)  \label{obj_eq_sep} 	
\end{align}
with the restrictions \eqref{Ustuff} and \eqref{Vstuff}.  Then, $\{\hat{\S}_{\bigcdot\bigcdot}^{(k)}\}_{k=1}^K$ solves \eqref{obj_eq}, where $\hat{\S}_{\bigcdot\bigcdot}^{(k)}=\hat{\U}_{\bigcdot}^{(k)} \hat{\V}_{\bigcdot}^{(k) T}$ for $k=1,\hdots,K$.        
\end{prop} 

From~\eqref{obj_eq_sep}, it is apparent that our objective gives the mode of a Bayesian posterior with normal priors on the errors and the factorization components, as stated in Proposition~\ref{prop4}.  
\begin{prop}
\label{prop4}
Let the entries of $\E_{\bigcdot \bigcdot}$ be independent \mbox{Normal}$(0,1)$, the entries of $\U_{i}^{(k)}$ be independent $\mbox{Normal}(0,\tau^2)$ if $\R[i,k]=1$, and the entries of          $\V_{j}^{(k)}$ be independent $\mbox{Normal}(0,\tau_k^2)$ if $\C[j,k]=1$, where $\tau_k^2=1/\lambda_k$.  Then,~\eqref{obj_eq_sep} is proportional to the log of the joint likelihood \[p\left(\X_{\bigcdot \bigcdot},\{\U_{\bigcdot}^{(k)}\}_{k=1}^K,\{\V_{\bigcdot}^{(k)}\}_{k=1}^K \mid \R, \C\right).\]  
\end{prop}

\section{Missing data imputation}
\label{missing}
The probabilistic formulation of the objective described in Section~\ref{bayesian} motivates a modified Expectation-Maximization (EM)-algorithm approach to impute missing data.  Let $\mathcal{M}$ index observations in the full dataset $\X_{\bigcdot \bigcdot}$ that are unobserved: $\mathcal{M}=\{(m,n):\X_{\bigcdot \bigcdot}[m,n] \text{ is missing}\}$.  Our iterative algorithm for missing data imputation proceeds as follows:
\begin{enumerate}
\item Initialize $\hat{\X}_{\bigcdot \bigcdot}$ by
\[\hat{\X}_{\bigcdot \bigcdot}[m,n] = \begin{cases}
 \X_{\bigcdot \bigcdot}[m,n]	 \text{ if } (m,n) \notin \mathcal{M} \\
 0	 \text{ if } (m,n) \in \mathcal{M}
 \end{cases}
\] 	
\item M-step: Estimate $\SolHat$ by optimizing \eqref{obj_eq} for $\hat{\X}_{\bigcdot \bigcdot}$.
\item E-step: Update $\hat{\X}_{\bigcdot \bigcdot}$ by $\hat{\X}_{\bigcdot \bigcdot}[m,n] = \sum_{k=1}^K \hatSk[m,n]$. 
\item Repeat steps 2.\ and 3.\ until convergence.
\end{enumerate}
  
Analogous approaches to imputation for other low-rank factorization techniques have been proposed \citep{kurucz2007methods,o2019linked,park2019integrative,mazumder2010spectral}.  Due to centering, initializing missing data to $0$ in step 1.\ is equivalent to starting with mean imputation, which is used by other SVD-based imputation approaches \citep{mazumder2010spectral,kurucz2007methods}; however, in practice random initializations can also be used.  The M-step maximizes the joint density for the model in Proposition~\ref{prop4}, where $\hat{\S}_{\bigcdot\bigcdot}^{(k)}=\hat{\U}_{\bigcdot}^{(k)} \hat{\V}_{\bigcdot}^{(k) T}$ for $k=1,\hdots,K$.  The E-step updates the log joint density with its conditional expectation over $\{\X_{\bigcdot \bigcdot}[m,n]: (m,n) \in \mathcal{M}\}$, by an argument analogous to that in \citet{zhang2005using}.  Thus the approach is an EM-algorithm for maximum a posteriori imputation under the model in Section~\ref{bayesian}, and it is also a direct block coordinate descent of objective~\eqref{obj_eq} over $\Sol$ and $\{\X_{\bigcdot \bigcdot}[m,n]: (m,n) \in \mathcal{M}\}$.      
  Crucially for our  context, the method can be used to impute data that may be missing from an entire column or an entire row of each $\X_{ij}$, or in certain cases can even be used to impute an entire matrix $\X_{ij}$ based on joint structure.    

\section{Application to TCGA data}
\label{application}

\subsection{Data acquisition and preprocessing}
\label{data}

Our data were curated for the TCGA Pan-Cancer Project and were used for the pan-cancer clustering analysis described in \cite{hoadley2018cell}.    We used data from four ($I=4$) omics sources: (1) batch corrected RNA-Seq data capturing (mRNA) expression for $20531$ genes, (2) batch corrected miRNA-Seq data capturing expression for $743$ miRNAs, (3) between-platform normalized data from the Illumina 27K and 450K platforms capturing DNA methylation levels for $22601$ CpG sites, and (4) batch-corrected reverse-phase protein array data capturing abundance for $198$ proteins.  These data are available for download at \url{https://gdc.cancer.gov/about-data/publications/PanCan-CellOfOrigin} [accessed 11/19/2019].  We consider data for $N=6,973$ tumor samples from different individuals with all four omics sources available; these tumor samples represent $J=29$ different cancer types, listed in Table~\ref{canc_types}.      

\begin{table}[htbp]
\centering
\caption{TCGA acronyms for the 29 different cancer types considered.}\label{canc_types}
\begin{tabular}{|c p{5cm}| c p{5cm}|}
\hline
Acronym & Cancer type & Acronym & Cancer type \\
  \hline
ACC	&Adrenocortical carcinoma & BLCA&	Bladder urothelial carcinoma   \\
BRCA&	Breast invasive carcinoma & CESC&	Cervical carcinoma \\
CHOL&	Cholangiocarcinoma & CORE&	Colorectal adenocarcinoma  \\
DLBC&	Diffuse large B-cell lymphoma  & ESCA&	Esophageal carcinoma \\
HNSC&	Head/neck squamous cell  &  KICH&	Kidney chromophobe\\
KIRC&	Kidney renal clear cell  & KIRP&	Kidney renal papillary cell   \\
LGG&	Brain lower grade glioma &  LIHC&	Liver hepatocellular carcinoma \\
LUAD&	Lung adenocarcinoma &  LUSC&	Lung squamous cell carcinoma  \\
MESO&	 Mesothelioma & OV&		Ovarian cancer \\
PAAD&	Pancreatic adenocarcinoma & PCPG&	Pheochromocytoma and paraganglioma  \\
PRAD&	Prostate adenocarcinoma & SARC&	Sarcoma \\
SKCM&	Skin cutaneous melanoma & STAD&	Stomach adenocarcinoma \\
TGCT&	Testicular germ cell tumors & THCA&	Thyroid carcinoma \\
THYM&	Thymoma & UCEC&	Uterine corpus endometrial carcinoma \\
UCS&	Uterine carcinosarcoma & &\\
\hline
        \end{tabular}
\end{table}

We log-transformed the counts for the RNA-Seq and miRNA-Seq sources. To remove baseline differences between cancer types, we center each data source to have mean $0$ across all rows for each cancer type:
\[\mbox{mean}(\X_{ij}[m,\bigcdot]) =0 \; \text{ for all } i, j, m.\]   
We filter to the $1000$ genes and the $1000$ methylation CpG probes that have the highest standard deviation after centering, leaving $M_1=1000$ genes, $M_2=743$ miRNAs, $M_3=1000$ CpGs, and $M_4=198$ proteins.  Lastly, to account for differences in scale, we standardize so that each variable has standard deviation $1$:
 \[\mbox{SD}(\X_{i\bigcdot}[m,\bigcdot]) =1 \; \text{ for all } i, m.\] 
This is a conservative alternative to scaling by an estimate of the noise variance, as mentioned in Section~\ref{scaling}.
 
 \subsection{Factorization results}
\label{facdata}

We apply the BIDIFAC+ method to the complete-case data with $I=4$ omics sources and $J=29$ cancer types.  We simultaneously estimate a maximum of $K=50$ low-rank modules; all modules are  non-zero, but the variation explained by the smaller modules are negligible. Figure~\ref{modulevar} gives the total variance explained by each module, $||\hatSk||_F^2$, for $k=1,\hdots,50$ in decreasing order.  The top $15$ modules ordered by total variance explained are given in Table~\ref{modules}, and all $50$ modules are given in the Supplemental Spreadsheet S1. %\url{http://www.ericfrazerlock.com/BIDIFAC_modules.xlsx}.  
The first module explains global variation, with all cancer types and all omics sources included.  Other modules that explain substantial variability across all or almost all cancer types are specific to each omics source: miRNA (Module 2), methylation (Module 3), gene expression (Module 5) and Protein (Module 8).

\begin{figure}[!h]
\label{modulevar}
\includegraphics[scale=0.63]{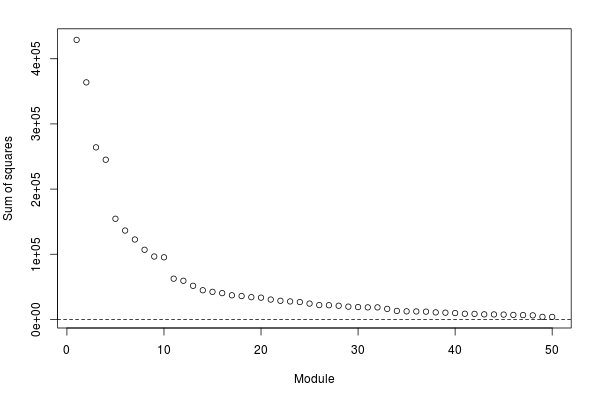}
\caption{Total sum of squared entries in each of the $50$ modules, ordered from largest to smallest.}
\end{figure}

\begin{table}[htbp]
\centering
\caption{Cancer types and sources for the first $15$ modules, ordered by variation explained.}\label{modules}
\begin{tabular}{|cc|p{7.8cm}| p{3.8cm}|}
\hline
Module & Rank & Cancer types & Omics sources \\
  \hline
1& 37&All cancers & mRNA miRNA Meth Protein  \\
2&	25 &All cancers & miRNA \\
3& 22	&BLCA BRCA CESC CHOL CORE DLBC ESCA HNSC LIHC LUAD LUSC OV PAAD PRAD SKCM STAD TGCT UCEC UCS & 	Meth  \\
4& 10	&ACC BLCA CHOL CORE DLBC ESCA HNSC KICH KIRC KIRP LGG LIHC LUAD LUSC MESO PAAD PCPG SARC SKCM STAD THCA THYM &	mRNA Meth\\
5&	24&All cancers  & mRNA\\
6&	17&BRCA & mRNA miRNA Meth Protein  \\
7&	15&LGG &  mRNA miRNA Protein \\
8&	20&All cancers *but* LGG &  Protein \\
9&	15 & THCA &		mRNA miRNA Protein \\
10&	20 &All cancers *but* LGG and TGCT & miRNA \\
11&	15 &CHOL KIRC KIRP LIHC & mRNA miRNA Meth Protein \\
12&	34 &LGG & Meth \\
13&	20 &BLCA CESC CORE ESCA HNSC LUSC SARC STAD & mRNA miRNA Meth Protein \\
14&	8 &KICH KIRC KIRP & mRNA miRNA Protein\\
15&	21 &BLCA BRCA CESC CHOL ESCA HNSC LUAD LUSC PAAD PRAD SKCM STAD TGCT UCEC UCS& mRNA miRNA\\
\hline
        \end{tabular}
\end{table}

The module that explains the fourth most variation (Module 4) identifies structure in the genes and DNA methylation that explains variation in $22$ of the $29$ cancer types; we focus on this module as an illustrative example.  The cancer types *not* included in Module 4 are BRCA (breast), CESC (cervical), OV (ovarian), PRAD (prostate), TGCT (testicular), UCEC (uterine endometrial), and UCS (uterine).  Interestingly, all tumor types that were excluded were cancers specific to either males or females (or heavily skewed in BRCA); while cancer types included have both sexes.  Figure~\ref{fig:gender} shows that Module 4 is indeed dominated by a single component that corresponds to molecular differences between the sexes. The gene loadings for this component are negligible except for those on the Y chromosome and two genes on the X chromosome that are responsible for X-inactivation in females: \emph{XIST} and \emph{TSIX}; the methylation loadings are negligible except for those in the X chromosome.  These results are an intuitive illustration of the method, revealing a multi-omic molecular signal that explains heterogeneity in some cancer types, but not all cancer types (only those that have both males and females).  

\begin{figure}[!h]
\includegraphics[scale=0.60]{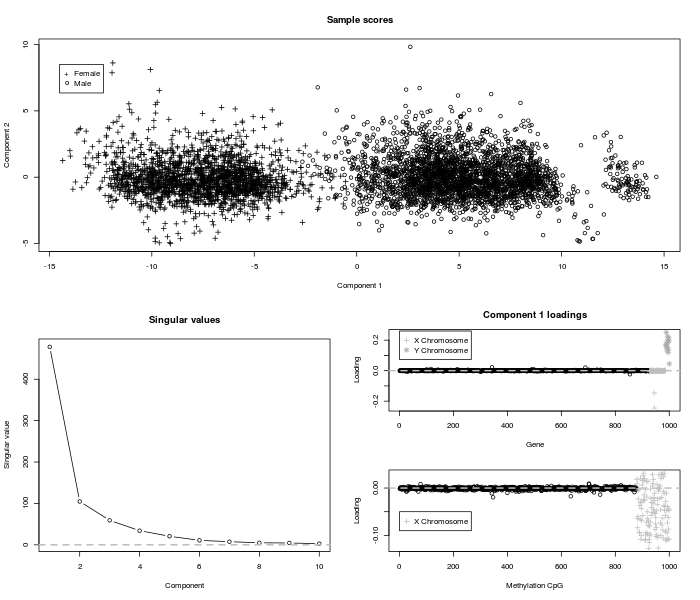}
\caption{For Module 4: Sample scores for the first two components (top), scree plot of singular values (bottom left), and loadings on genes and methylation CpGs (bottom right) for the first component.  This module includes $22$ cancer types with samples from both sexes, and it is dominated by molecular signals that distinguish males from females.}
\label{fig:gender}
\end{figure}

The module that explains the sixth most variation (Module 6) identifies structure across all four omics sources that explains variation in the breast cancer (BRCA) samples only.  Figure~\ref{fig:breast} shows that the first two components in this module are driven primarily by distinctions between the PAM50 molecular subtypes for BRCA \citep{cancer2012comprehensive}.  Thus, our analysis suggests that molecular signals that distinguish these subtypes are present across all four omics sources, but that these signals do not explain substantial variation within any other type of cancer considered. 

\begin{figure}[!h]
\includegraphics[scale=0.60]{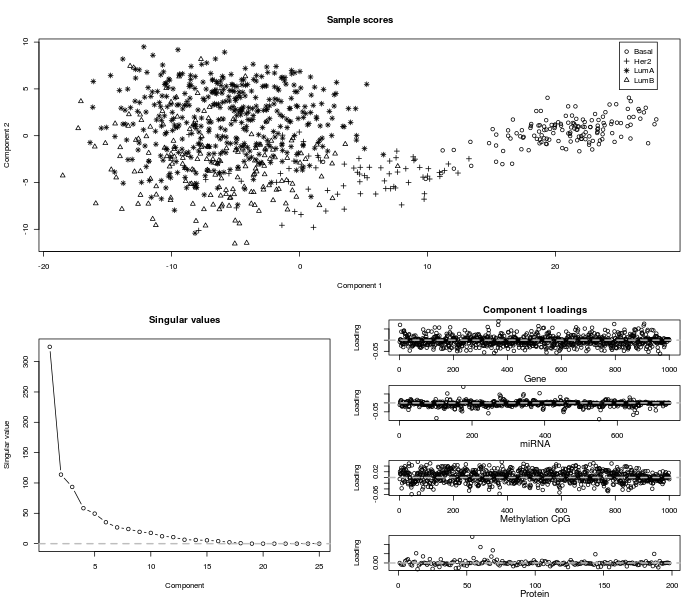}
\caption{For Module 6: Sample scores for the first two components (top), scree plot of singular values (bottom left), and loadings for all four omics platforms  (bottom right).  This module includes only breast (BRCA) tumor samples, and it is dominated by molecular signals that distinguish the PAM50 subtypes (Bsasal, Her2, LumA and LumB).}
\label{fig:breast}
\end{figure}

Several other modules explain variability in just one type of type of cancer, including LGG (Module 7: mRNA, miRNA and Protein), THCA (Module 9), UCEC (Module 16), and PRAD (Modules 18 and 19). Module 12, which is specific to LGG methylation, reveals distinct clustering by mutation status of the \emph{IDH} genes (see Figure~\ref{fig_lgg}). IDH mutations have been shown to lead to a distinct CpG-island hyper-methylated phenotype \citep{noushmehr2010identification}. Other modules explain variability in multiple cancer types that share similarities regarding their origin or histology.  For example, Module 14 explains variability within the three kidney cancers (KICH, KIRC, and KIRP), and digestive and gastrointestinal cancers (CORE, ESCA, PAAD, STAD) are represented in Modules 25 (methylation) and 28 (mRNA).  %[KATIE: would be great o have your input here on what else we can mention or investigate further].   
\begin{figure}[!h]
\centering
\includegraphics[scale=0.67]{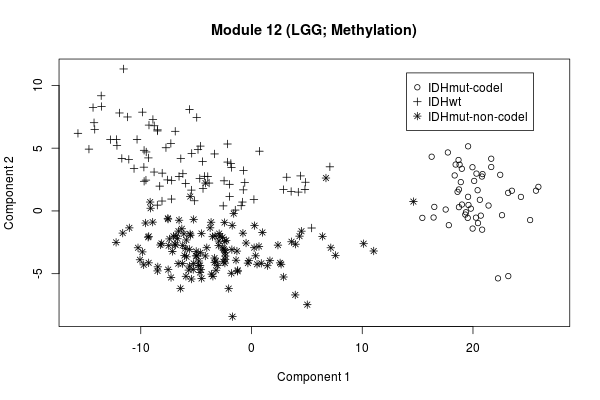}
\caption{Scores for the first two components of Module 12 (LGG; methylation), with symbols and colors showing separation by IDH mutation status (wild-type, IDH mutant, and IDH mutant with co-deletion).}
\label{fig_lgg}
\end{figure}

To assess sensitivity of the results to the maximum number of modules $\tilde{K}=50$, we also performed the decomposition with $\tilde{K}=25$, $30$, $35$, $40$, or $45$.  The resulting decompositions share several similarities. Nine of the $15$ modules in Table~\ref{modules} have a precise match in each of the $6$ decompositions (Modules 1, 2, 4, 5, 6, 7, 9, 12, 14).  Entrywise correlations for the terms $\Sk$ between matched modules ranged from $r=0.88$ to $r=0.99$.  Other modules had slight variations to the cancer types and omics sources included. The modules for different $\tilde{K}$ are provided as separate tabs in Supplemental Spreadsheet S1.%
 %\url{http://www.ericfrazerlock.com/BIDIFAC_modules.xlsx}.  
 The overall structural approximations $\sum_{k=1}^{\tilde{K}} \Sk$ were very similar across decompositions, with entrywise correlations all greater than $r=0.99$.

 \subsection{Missing data imputation}
\label{misdata}
To assess the accuracy of missing data imputation using BIDIFAC+, we hold-out observed entries, rows, and columns of each dataset in the pan-omics pan-cancer and impute them using the approach in Section~\ref{missing}.  We randomly set $100$ columns (samples) to missing for each of the $4$ omic platforms, and we randomly set $100$ rows (features) to missing for each of the  $29$ cancer types.  We then randomly set $5000$ of the values remaining in the joint matrix $\X_{\bigcdot \bigcdot}$ to missing.  We impute missing values using BIDIFAC+ as described in Section~\ref{missing}, and for comparison we use an analogous approach to imputation using four other low-rank factorizations: (1) soft-threshold (nuclear norm) SVD of the joint matrix $\X_{\bigcdot \bigcdot}$, (2) soft-threshold SVD of each matrix $\X_{ij}$ separately, (3) hard-threshold SVD (SVD approximation using the first $R$ singular values) of $\X_{\bigcdot \bigcdot}$, (4) hard-threshold SVD of each $\X_{ij}$ separately.  For the soft-thresholding SVD methods, the penalty factor is estimated by random matrix theory as in Section~\ref{objective}.  For the hard-thresholding methods the ranks are determined by cross-validation by minimizing imputation error on an additional held-out cohort of the same size.  

We consider the imputation error under the different methods, broken down by (1) observed values, (2) values that are missing but have the rest of their row and column present (entrywise missing), (3) values that are missing their entire row, (4) values that are missing their entire column, and (5) values that are missing both their row and their column.  For a given set of values $\mathcal{M}$, we compute the relative squared error as
\[\mbox{RSE} = \frac{\sum_{(m,n) \in \mathcal{M}} (\X_{\bigcdot \bigcdot}[m,n]-\hat{\X}_{\bigcdot \bigcdot}[m,n])^2}{\sum_{(m,n) \in \mathcal{M}} \X_{\bigcdot \bigcdot}[m,n]^2},\]
where $\hat{\X}_{\bigcdot \bigcdot}$ is the structural approximation resulting from the given method.  Table~\ref{tbl:impute_acc} gives the RSE for each method and for each missing condition.  Imputation by BIDIFAC+ outperforms the other methods for each type of missingness, illustrating the advantages of decomposing joint and individual structures. The hard-thresholding approaches have much less error for the observed data than for the missing data, due to over-fitting of the signal. 

 \begin{table}[htbp]
\centering
\caption{Imputation RSE under different approaches and different types of missingness.}\label{tbl:impute_acc}
\begin{tabular}{|r|c c c c c|}
  \hline
  Method & Observed & Entrywise & Row & Column & Both \\
\hline
 BIDIFAC+ & $0.510$ & $0.558$ & $0.670$ &  $0.807$ & $0.881$  \\
 Soft-SVD (joint) & $0.531$ & $0.621$ & $0.678$ & $0.834$ & $0.894$ \\
Soft-SVD (separate) & $0.564$ & $0.610$ & $1.000$ & $1.000$ & $1.000$  \\
Hard-SVD (joint) & $0.431$ & $0.559$ & $0.829$ & $0.908$ & $1.200$  \\
Hard-SVD (separate) & $0.344$ & $0.581$ & $1.000$ & $1.000$ & $1.000$  \\

        \hline
        \end{tabular}
\end{table}
 
 \section{Simulation studies}
\label{simulation}

\subsection{Vertically linked simulations}
\label{vert_sim} 
We conduct a simulation study to assess the accuracy of the BIDIFAC+ decomposition in the context of vertical integration, where there is a single shared column set $(J=1)$.  For all scenarios, we simulate data according to model~\eqref{slide} wherein the entries of the residual noise $\E_{\bigcdot \bigcdot}$ are generated independently from a Normal$(0,1)$ distribution and the entries of each $\U_i^{(k)}$ and $\V^{(k)}$ are generated independently from a Normal$(0,\sigma^2)$ distribution. 

We first consider a scenario with $I=3$ matrices, each of dimension $100 \times 100$ ($N=100$ and $M_1=M_2=M_3=100$), with low-rank modules that are shared jointly, shared across each pair of matrices, and individual to each matrix:
\begin{align} \R = \left [ \begin{array}{c c c c c c c}   1 & 1 & 1 & 0 & 1 & 0 & 0  \\  1 & 1 & 0 & 1 & 0 & 1 & 0 \\ 1 & 0 & 1 & 1 & 0 & 0 & 1 \end{array} \right ] . \label{eqR} \end{align}
We consider a ``low-rank" and a ``high-rank" condition across three different signal-to-noise levels. For the low-rank condition, each of the seven modules has rank $R=1$; for the high-rank condition, each module has rank $R=5$.  The variance of the factorized signal component, $\sigma^2$ is set to be $\sqrt{1/2}, 1,$ or $\sqrt{10}$, so that the signal-to-noise ratio (s2n) of each components is $1/2$, $1$, or $10$, respectively.

For each condition, we apply four approaches to uncover the underlying decomposition:
\begin{enumerate}
\item BIDIFAC+, with $\R$ given by~\eqref{eqR}, as in the true generative model,
\item BIDIFAC+, with $\R$ estimated,
\item SLIDE, with $\R$ and the true ranks of each module ($R=1$ or $R=5$) provided,
\item SLIDE, with $\R$ and the ranks of each module estimated via the default cross-validation scheme.
\end{enumerate}
We use SLIDE as the basis of comparison with BIDIFAC+, because it is the only other method that is designed to recover each term in the decomposition and it generally outperforms other vertically linked decomposition methods \citep{gaynanova2017structural, park2019integrative}.
For each case we compute the mean relative squared error (RSE) in recovering each module of the decomposition:
\begin{align} \text{RSE} = \frac{1}{K} \sum_{k=1}^K \frac{||\Sk - \hatSk||_F^2}{||\Sk||_F^2}. \label{RSE} \end{align}
The mean RSE for each condition and under each approach is shown in Table~\ref{tbl:bidi.slide}, broken down by the global module, pairwise modules, and individual modules.  BIDIFAC+ generally outperforms or performs similarly to SLIDE, even when the true ranks are used for the SLIDE implementation (the ranks are never fixed for BIDIFAC+). An exception is when the ranks and s2n ratio are small (rank=1, s2n=0.5), where BIDIFAC+ tends to over-shrink the signal.  BIDIFAC+ performs particularly well relative to SLIDE when the rank is large and s2n is high.  One likely reason for this improvement is that the SLIDE model necessarily restricts the factorized components $\U_i^{(k)}$ and $\V^{(k)}$ to be mutually orthogonal, whereas BIDIFAC+ has no such constraint.  This restriction can be limiting when decomposing generated signals that are independent but not orthogonal \citep{park2019integrative}.  Moreover, when estimating the ranks the SLIDE model can drastically underperform relative to using the true ranks. The results for BIDIFAC+ when fixing the true modules $\R$ vs. estimating $\R$ are nearly identical; because all possible modules are present for this scenario, the two approaches are very similar despite subtle differences in the algorithms.

We consider another scenario with a larger number of matrices ($I=10$), each of dimension $100 \times 100$ ($N=100$, $M_1=\cdots=M_{10}=100)$ and sparsely distributed modules.  We generate $10$ low rank modules out of $2^{10}-1=1023$ possibilities, that are present on (1) $X_{11}$ only, (2) $X_{11}$ and $X_{21}$, (3) $X_{11}$, $X_{21}$, and $X_{31}$, etc.  We again consider low-rank ($R=1$) and high-rank ($R=5$) scenarios for all modules, and three signal-to-noise levels $0.5, 1,$ and $10$.  The resulting mean RSE~\eqref{RSE} over all modules, for each approach, is shown in Table~\ref{tbl:bidi.slide2}.  Here, BIDIFAC+ with fixed true $\R$ generally performs better than estimating $\R$; however, these gains are modest for most scenarios, suggesting the BIDIFAC+ generally does a good job of identifying which of the $1023$ possible modules are non-zero.

\begin{table}[htbp]
\caption{Comparison of BIDFAC+ and SLIDE signal decomposition RSE ($I=3$ sources).}\label{tbl:bidi.slide}
\begin{tabular}{|cl|rr|rr|}
  \hline
  & & \multicolumn{2}{c|}{BIDIFAC+} & \multicolumn{2}{c|}{SLIDE} \\
Scenario &  Structure &  True $\R$ & Estimated $\R$  & True ranks  & Estimated ranks \\
\hline
  & Global & $0.130$ & $0.130$& $0.120$ & $0.120$   \\
Rank=$1$, s2n=$0.5$        & Pairwise & $0.157$ & $0.156$& $0.103$ & $0.103$ \\
        & Individual &  $0.197$ & $0.197$& $0.118$ & $0.118$  \\
  \hline
  & Global & $0.060$ & $0.060$& $0.084$ & $0.084$   \\
Rank=$1$, s2n=$1$        & Pairwise & $0.068$ & $0.068$& $0.053$ & $0.053$ \\
        & Individual &  $0.070$ & $0.070$& $0.048$ & $0.048$  \\
        \hline
& Global & $0.010$ & $0.010$& $0.035$ & $3.65$   \\
Rank=$1$, s2n=$10$        & Pairwise & $0.005$ & $0.005$& $0.027$ & $1.00$ \\
        & Individual &  $0.008$ & $0.008$& $0.037$ & $0.689$  \\
          \hline
            & Global & $0.270$ & $0.270$& $0.276$ & $0.869$   \\
Rank=$5$, s2n=$0.5$        & Pairwise & $0.268$ & $0.268$& $0.263$ & $0.460$ \\
        & Individual &  $0.329$ & $0.329$& $0.306$ & $0.317$  \\
  \hline
  & Global & $0.123$ & $0.123$& $0.232$ & $1.320$   \\
Rank=$5$, s2n=$1$        & Pairwise & $0.121$ & $0.121$ & $0.189$ & $0.674$ \\
        & Individual &  $0.148$ & $0.148$& $0.241$ & $0.485$  \\
        \hline
& Global & $0.080$ & $0.080$& $0.233$ & $2.36$   \\
Rank=$5$, s2n=$10$        & Pairwise & $0.060$ & $0.060$& $0.189$ & $0.917$ \\
        & Individual &  $0.089$ & $0.089$& $0.249$ & $0.703$  \\
          \hline
        \end{tabular}
\end{table}
% softImpute & Individual      & 85 & 0.53 &       0.001\\ \hline
\begin{table}[htbp]
\centering
\caption{Comparison of BIDFAC+ and SLIDE signal decomposition RSE ($I=10$ sources).}\label{tbl:bidi.slide2}
\begin{tabular}{|cl|rr|rr|}
  \hline
  & & \multicolumn{2}{c|}{BIDIFAC+} & \multicolumn{2}{c|}{SLIDE} \\
Ranks &  s2n &  True $\R$ & Estimated $\R$  & True ranks  & Estimated ranks \\
\hline
 $1$ & $0.5$ & $0.150$ & $0.150$ &  $0.116$ & $0.116$  \\
 $1$ & $1$ & $0.076$ & $0.078$ & $0.105$ & $0.105$ \\
 $1$ & $10$ & $0.032$ & $0.025$ & $0.060$ & $0.060$  \\
 $5$ & $0.5$ & $0.297$ & $0.320$ & $0.402$ & $0.685$  \\
 $5$ & $1$ & $0.177$ & $0.189$ & $0.324$ & $0.603$  \\
 $5$ & $10$ & $0.167$ & $0.245$ & $0.347$ & $0.347$  \\
        \hline
        \end{tabular}
\end{table}

\subsection{Missing data simulation}
\label{missing_sim} 

Here we assess the performance of missing data imputation for a $3 \times 3$ grid of matrices ($I=J=3$) with $\X_{ij}: 100 \times 100$ for each $i,j$.  Data $\X_{\bigcdot \bigcdot}$ are generated as in model~\eqref{bidiflex}, with one fully shared module $(\R[\bigcdot,1]= \C[\bigcdot,1]=[1,1,1])$ and modules specific to each of the $9$ matrices $(\R[\bigcdot,1]= \C[\bigcdot,1]=[1,0,0]$, etc.).  All modules have rank $5$.  The entries of residual noise are generated independently from a Normal$(0,1)$ distribution, and the entries of each $\U_i^{(k)}$ and $\V_i^{(k)}$ are generated independently from a Normal$(0,\sigma_k^2)$ distribution.  We consider different levels of the joint signal strength $\sigma_j^2:=\sigma_1^2$ and the matrix-specific (i.e., individual) signal strength $\sigma_i^2:=\sigma_2^2=\cdots=\sigma_{10}^2$.  We further consider scenarios with different kinds of missingness: (1) $1/9$ of the entries in $\X_{\bigcdot \bigcdot}$ missing at random, (2) $1/9$ of the columns in each $\X_{ij}$ entirely missing at random, and (3) one of the nine datasets $\X_{ij}$ entirely missing. 

For each scenario, we impute missing values using the same approaches used in Section~\ref{misdata}.  For hard SVD imputation we use the true rank of the underlying structure. We consider three versions of BIDIFAC+ imputation: (1) initializing by setting missing values to $0$, (2) initializing by generating missing values randomly from a $N(0,1)$ distribution, and (3) initializing at $0$ and fixing the modules $\R$ and  $\C$ to match the data generation (i.e., one fully shared module and one module specific to each matrix).  The resulting RSEs for the missing signal,
\[\mbox{RSE} = \frac{\sum_{(m,n) \in \mathcal{M}} \left(\sum_{k=1}^{10}\S^{(k)}_{\bigcdot \bigcdot}[m,n]-\hat{\X}_{\bigcdot \bigcdot}[m,n]\right)^2}{\sum_{(m,n) \in \mathcal{M}} \left(\sum_{k=1}^{10}\S_{\bigcdot \bigcdot}^{(k)}[m,n]\right)^2},\] 
are shown in Table~\ref{tbl:sim:impute_acc}, averaged over $100$ replications.   Imputation by BIDIFAC+ is flexible and performs relatively well across the different scenarios.  As expected, the accuracy of column-wise and block-wise imputation depends strongly on the relative strength of the joint signal.  Initializing missing values to $\0$ or randomly generated values gave very similar results, suggesting that either approach is reasonable in practice.  BIDIFAC+ using the true $\R$ and $\C$ is comparable to  estimating $\R$ and $\C$ across most scenarios.  An exception is for block-missing data (i.e., missing an entire $\X_{ij}$);  in this context, there is some ambiguity on whether shared signals defined by $\hat{\R}$ and $\hat{\C}$ include the missing block, and so the performance with the true $\R$ and $\C$ is slightly better.  Also, as for other blocks, values for the missing block are imputed under the assumption of error variance $1$ and mean $0$; in practice they cannot be transformed back to their original measurement scale without prior knowledge of the mean and error variance.

 \begin{table}[htbp]
\centering
\caption{Imputation RSE under different approaches, for different types of missingness and joint and individual signal strengths ($\sigma^2_j$ and $\sigma^2_i$ respectively).}\label{tbl:sim:impute_acc}
\begin{tabular}{|r|c c c |c c c| c c c|}
\hline
  & \multicolumn{3}{c|}{Entry missing} & \multicolumn{3}{c|}{Column missing} & \multicolumn{3}{c|}{Block missing} \\
  Joint signal ($\sigma^2_j$) & $\sqrt{10}$ & 1 & 1 & $\sqrt{10}$ & 1 & 1 & $\sqrt{10}$ & 1 & 1 \\
  Individual signal ($\sigma^2_i$) & 1 & 1 & $\sqrt{10}$ & 1 & 1 & $\sqrt{10}$ & 1 & 1 & $\sqrt{10}$ \\
  \hline 
   BIDIFAC+ & 0.01 & 0.06 & 0.01 & 0.11 & 0.56 & 0.93 & 0.16 & 0.57 & 0.94\\
   BIDIFAC+ (random init) & 0.01 & 0.06 & 0.01 & 0.11 & 0.56 & 0.93 & 0.16 & 0.57 & 0.94\\
   BIDIFAC+ (true $\R$, $\C$) & 0.01 & 0.06 & 0.01 & 0.10 & 0.55 & 0.92 & 0.10 & 0.52 & 0.93\\
    Soft-SVD (joint) & 0.03 & 0.15 & 0.04 & 0.11 & 0.58 & 0.94 & 0.10 & 0.53 & 0.93 \\
Soft-SVD (separate)& 0.02 & 0.09 & 0.02 & 1.00 & 1.00 & 1.00 & 1.00 & 1.00 & 1.00\\
Hard-SVD (joint) & 0.01 & 0.06 & 0.01 & 0.99 & 2.02 & 0.99 & 1.26 & 2.15 & 1.04\\
Hard-SVD (separate) & 0.01 & 0.03 & 0.01 & 1.00 & 1.00 & 1.00 & 1.00 & 1.00 & 1.00\\
\hline

        \hline
        \end{tabular}
\end{table}

\subsection{Application-motivated simulation}
\label{app_sim}

Here we assess the recovery of the underlying structure and the accuracy of the decomposition into shared components for a bidimensionally linked scenario that reflects our motivating application in Section~\ref{application}.  We generate data by taking the estimated decomposition from Section~\ref{facdata} and adding independent noise to it.  That is, we simulate \[\tilde{\X}_{\bigcdot \bigcdot} = \sum_{k=1}^{50} \alpha \hatSk + \tilde{\E}_{\bigcdot \bigcdot}\] where $\SolHat$ is the estimated decomposition from Section~\ref{facdata}, the entries of $\tilde{\E}_{\bigcdot \bigcdot}$ are independent $\mbox{Normal}(0,1)$, and $\alpha>0$ is a parameter that controls the total signal-to-noise ratio.  We consider three total signal-to-noise ratios, defined by \[\mbox{s2n}=\mbox{var}(\sum_{k=1}^{50} \alpha \hatSk)/\mbox{var}(\tilde{\E})= \mbox{var}(\sum_{k=1}^{50} \alpha \hatSk),\]
$\mbox{s2n}=0.2,0.5,$ and $5$.  The scenario with $\mbox{s2n}=0.5$ corresponds most closely to the real data, for which the ratio of the estimated signal variance over the residual variance is $0.552$.  For each scenario, we estimate the underlying decomposition using BIDIFAC+ with the true $\R$ and $\C$ fixed, and using BIDIFAC+ with estimated modules $\tilde{\R}$ and $\tilde{\C}$ and $\tilde{K}=50$.  In each case, we compute the RSE as follows
   \begin{align} \text{RSE} = \frac{1}{50} \sum_{k=1}^{50} \frac{||\tilSk - \alpha \hatSk||_F^2}{||\alpha\Sk||_F^2}. \label{RSE2}. \end{align}
   When computing RSE, we permute the $50$ modules so that $\tilde{\R}[\bigcdot,k]=\R[\bigcdot,k]$ and $\tilde{\C}[\bigcdot,k]=\C[\bigcdot,k]$ wherever possible, and set $\tilSk=\0$ if $\tilde{\R}[\bigcdot,k]\neq\R[\bigcdot,k]$ and $\tilde{\C}[\bigcdot,k]\neq\C[\bigcdot,k]$.   We also compute the relative overall signal recovery (ROSR) as 
     \begin{align} \text{ROSR} =  \frac{||\sum_{k=1}^K \hatSk - \sum_{k=1}^K\alpha \hatSk||_F^2}{||\sum_{k=1}^K \alpha\Sk||_F^2}. \label{ROSR} \end{align} 
     
     The results are shown in Table~\ref{tbl:rec_acc}, and demonstrate that the underlying decomposition is recovered reasonablly well in most scenarios.  However, the RSE for estimated modules is often substantially more than the RSE using the true modules, as the row and column sets defining the modules can be estimated incorrectly.  Moreover, the overall signal recovery error (ROSR) is generally substantially less than the mean error in recovering each module (RSE), demonstrating how the decomposition can be estimated incorrectly even if the overall signal is estimated with high accuracy.
      \begin{table}[htbp]
\centering
\caption{Relative squared error of the decomposition (RSE) and relative overall signal recovery (ROSR) using BIDIFAC+ with known modules ($\R$ and $\C$) and estimated modules ($\hat{\R}$ and $\hat{\C}$). }\label{tbl:rec_acc}
\begin{tabular}{|r|c c c c|}
  \hline
  s2n & RSE($\R, \C$) & RSE ($\tilde{\R}, \tilde{\C}$) & ROSR ($\R, \C$) & ROSR ($\tilde{\R}, \tilde{\C}$)   \\
\hline
$0.2$ & $0.356$ & $0.531$ & $0.170$ &  $0.189$ \\
$0.5$ & $0.242$ & $0.386$ & $0.131$ & $0.143$ \\
$5$ & $0.128$ & $0.346$ & $0.012$ & $0.026$  \\

        \hline
        \end{tabular}
\end{table}

\section{Discussion}
\label{discussion}
The successful integration of multiple large sources of data is a pivotal challenge for many modern analysis tasks.  While several approaches have been developed, they largely do not apply to the context of bidimensionally linked matrices.  BIDIFAC+ is a flexible approach for dimension reduction and decomposition of shared structures among bidimensionally linked matrices, which is competitive with alternative methods that integrate over a single dimension (rows or columns).  Here we have focused primarily on the accuracy of the estimated decomposition and exploratory analysis of the results.  BIDIFAC+ may also be used for other tasks, such as missing data imputation or as a dimension reduction step preceding statistical modeling (e.g., as in principal components regression).  For these other tasks it is desirable to model statistical uncertainty, and fully Bayesian extensions that capture the full posterior distribution about the mode in Section~\ref{bayesian} are potentially very useful. Moreover, while we have explored the uniqueness of the decomposition under BIDFAC+, it is worthwhile to establish conditions that are both necessary and sufficient for its identifiability. 

An implicit assumption of the BIDIFAC+ framework is that shared structures are present over complete submatrices.  However, it is conceivable that  structured variation may take other partially shared forms. For example, a pattern of variation in DNA methylation may exist across several cancer types but only regulate gene expression in some of those cancer types.  Allowing for such shared structures, that do not exist over a complete submatrix, is an interesting direction of future work for which the separable form of the objective~\eqref{obj_eq_sep} may be useful.  Moreover, the approach may be extended to account for different layers of granularity in the row and column sets; for example, it would be interesting to also identify shared and specific patterns of variability among known subtypes within each cancer type.  We select penalty terms by random matrix theory, which requires assumptions of independent error and homogenous variance;  an alternative strategy is to use the imputation accuracy of held-out data as a metric for parameter selection. 
 
Computing time for BIDIFAC+ can range from $<5$ minutes for each simulation in Sections~\ref{vert_sim} and~\ref{missing_sim} to $\approx 24$ hours until convergence for the pan-omics pan-cancer application in Section~\ref{application} and the accompanying simulation in Section~\ref{app_sim}.  Thus, computational feasibility must be considered carefully or larger scale problems. 

Our application to pan-omics pan-cancer data from TCGA revealed molecular patterns that explain variability across all or almost all types of cancer, both across omics platforms and within each omics platform.  However, it also revealed patterns several instances in which patterns are specific to one or a small subset of cancers, and these often show sharp distinctions of previously known molecular subtypes (e.g., for BRCA and LGG). Interestingly, BRCA was the only tumor type that showed up with all four platforms in a module. Together, they strongly separated the Basal-like molecular subtype from other subtypes of breast cancer.  This mirrors the analysis of individual data types in \cite{cancer2012comprehensive}.  The LGG data also split by both histological groups and mutation status based on BIDFAC+, even though both were not included in the analysis. Module 7 included mRNA, miRNA, and protein and was predominantly driven by co-deletion of 1p/19q which is predominantly observed in oligodendrogliomas and is associated with better overall survival.  This mirrors the previous TCGA work that showed that the LGG could be predominately split by 1p/19q deletion, IHD1 status (Module 12, for methylation) or TP53 mutation status \citep{cancer2015comprehensive}.  An important insight provided by BIDIFAC+ is that these molecular distinctions are specific to BRCA and LGG, respectively, suggesting that similar phenomena do not account for heterogeneity within other types of cancers.  Other modules that are broadly shared across cancer types have potential to reveal relevant molecular signal that are undetectable within a single cancer type, especially those with smaller sample sizes. Beyond exploratory visualization, to systematically investigate the clinical relevance of the underlying modules, one can cluster their structure to identify novel subtypes analogous to the approach described in \citet{hellton2016integrative}. Furthermore, one can use the results in a predictive model for a clinical outcome, analogous to the approach described in \citet{kaplan2017prediction}.  We are currently pursuing the use of BIDIFAC+ results for these tasks.  

\section*{Availability}
    A GitHub repository for BIDIFAC+ is available at \url{https://github.com/lockEF/bidifac}, and R code for all analyses presented herein is provided in a supplemental zipped folder.
    
\section*{Acknowledgement}
    The authors gratefully acknowledge the support of NCI grant R21 CA231214, and the very helpful feedback of the Editor, AE, and four referees.

\appendix

\section{Module enumeration}
\label{ModEnum}

As the default representation of model (6) in the main article, set $K=(2^I-1)(2^J-1)$ and let $\R$ and $\C$ enumerate all possible modules as follows.  For $k=1,\hdots,K$, let $\R[\bigcdot,k]$ be the $I$-digit binary representation for $k \, \mbox{mod} \, (2^I-1)+1$, where $\mbox{mod}$ gives the modulo (remainder) operator.  For $k=1,\hdots,K$, let $\C[\bigcdot,k]$ give the $J$-digit binary representation for $\lceil k/(2^I-1) \rceil$, where $\lceil \cdot \rceil$ gives the ceiling operator.     

\section{Proofs} 
\label{proofs}

%\begin{prop} 
%\label{prop_nec}
%Under objective~\eqref{obj_eq}, the following are necessary to allow for each $\hatSk$ to be non-zero 
%\begin{enumerate}
%\item If for $k'\neq k$ the rows and columns of module $k'$ are contained within those for module $k$, $\R[i, k]-\R[i, k'] \geq 0$ $\forall$ $i$ and $\C[j, k]-\C[j, k'] \geq 0$ $\forall$ $j$, then $\lambda_k>\lambda_{k'}$.
%\item If $\mathcal{I}_k \subset \{1,\hdots,k-1,k+1,\hdots,K\}$ is any subset of modules that together cover the rows and columns of module $k$, $\sum_{j\in \mathcal{I}_k} \R[\bigcdot,j]= r \cdot \R[\bigcdot,k]$ and $\sum_{j\in \mathcal{I}_k} \C[\bigcdot,j]= c \cdot \C[\bigcdot,k]$ for positive integers $r$ and $c$, then $\lambda_k< \sum_{j\in \mathcal{I}_k} \lambda_{j}$.
%\end{enumerate}
%\end{prop}

\subsection{Proof of Proposition 1}

\begin{proof}
Let $\SolHat \in \SolSpace$ be a minimizer of the objective function $f(\bigcdot)$. Assume a violation of condition 1., wherein $\lambda_{k'}\geq \lambda_k$. Consider another minimizer $\SolTil$, where $\widetilde{\bS}_{\bigcdot\bigcdot}^{(k)}=\bzero$ and $\widetilde{\bS}_{\bigcdot\bigcdot}^{(k')}=\widehat{\bS}_{\bigcdot\bigcdot}^{(k)}+\widehat{\bS}_{\bigcdot\bigcdot}^{(k')}$, and all other modules are equal. Then, using the triangle inequality,
\begin{align*}
f(\SolHat)-f(\SolTil)&=\lambda_k ||\widehat{\bS}_{\bigcdot\bigcdot}^{(k)}||_* +\lambda_{k'} ||\widehat{\bS}_{\bigcdot\bigcdot}^{(k')}||_*-\lambda_k||\widehat{\bS}_{\bigcdot\bigcdot}^{(k)}+\widehat{\bS}_{\bigcdot\bigcdot}^{(k')}||_*\\
&\geq \lambda_k ||\widehat{\bS}_{\bigcdot\bigcdot}^{(k)}||_* +\lambda_{k'} ||\widehat{\bS}_{\bigcdot\bigcdot}^{(k')}||_*-\lambda_k(||\widehat{\bS}_{\bigcdot\bigcdot}^{(k)}||_*+||\widehat{\bS}_{\bigcdot\bigcdot}^{(k')}||_*)\\
&\geq \lambda_k ||\widehat{\bS}_{\bigcdot\bigcdot}^{(k)}||_* +\lambda_{k'} ||\widehat{\bS}_{\bigcdot\bigcdot}^{(k')}||_*-\lambda_k(||\widehat{\bS}_{\bigcdot\bigcdot}^{(k)}||_*)-\lambda_{k'}||\widehat{\bS}_{\bigcdot\bigcdot}^{(k')}||_*)\\
&=0,
\end{align*}
and thus there is a solution in which module $k$ is $\0$, regardless of the data $\X_{\bigcdot \bigcdot}$.

Now assume a violation of condition 2., wherein $\lambda_k\geq \sum_{j\in\mathcal{I}_k}\lambda_{j}$. Let $\widehat{\S}^{(k)}=\sum_{j\in\mathcal{I}_k} \widehat{\S}_{\bigcdot\bigcdot}^{j \prime}$, where $\widehat{\S}_{\bigcdot\bigcdot}^{j \prime}$ contains the submatrix of $\widehat{\S}_{\bigcdot\bigcdot}^{(k)}$ corresponding to $\R[\bigcdot,j]$ and $\C[\bigcdot,j]$ and $\0$ otherwise.  Consider another decomposition $\SolTil$, where $\widetilde{\S}_{\bigcdot\bigcdot}^{(k)}=\0$ and $\widetilde{\S}_{\bigcdot\bigcdot}^{(j)}=\widehat{\S}_{\bigcdot\bigcdot}^{(j)}+\widehat{\S}_{\bigcdot\bigcdot}^{(j)\prime}$ for all $j\in \mathcal{I}_k$. Then,
\begin{align*}
    f(\SolHat)-f(\SolTil)&=\lambda_k||\widehat{\bS}_{\bigcdot\bigcdot}^{(k)}||_*+\sum_{j\in\mathcal{I}_k}\lambda_j||\widehat{\bS}_{\bigcdot\bigcdot}^{(j)}||_*-\sum_{j\in\mathcal{I}_k}\lambda_j ||\widehat{\bS}_{\bigcdot\bigcdot}^{(j)}+\widehat{\bS}_{\bigcdot\bigcdot}^{(j)\prime}||_*\\
    &\geq \lambda_k||\widehat{\bS}_{\bigcdot\bigcdot}^{(k)}||_*+\sum_{j\in\mathcal{I}_k}\lambda_j||\widehat{\bS}_{\bigcdot\bigcdot}^{(j)}||_*-\sum_{j\in\mathcal{I}_k}\lambda_j ||\widehat{\bS}_{\bigcdot\bigcdot}^{(j)}||_*-\sum_{j\in\mathcal{I}_k} \lambda_j ||\widehat{\bS}_{\bigcdot\bigcdot}^{(j)\prime}||_*\\
    &=\lambda_k||\widehat{\bS}_{\bigcdot\bigcdot}^{(k)}||_*-\sum_{j\in\mathcal{I}_k} \lambda_j ||\widehat{\bS}_{\bigcdot\bigcdot}^{(j)\prime}||_*\\
    &\geq \lambda_k||\widehat{\bS}_{\bigcdot\bigcdot}^{(k)}||_*-\sum_{j\in\mathcal{I}_k} \lambda_j ||\widehat{\bS}_{\bigcdot\bigcdot}^{(k)}||_*\\
    &\geq 0,
\end{align*}
and thus there is a solution in which module k is 0, regardless of the data $\X_{\bigcdot \bigcdot}$
\end{proof}

\subsection{Proof of Proposition 4}

%\setcounter{prop}{3}
%\begin{prop}
%Setting $\lambda_k=\sqrt{\R[\bigcdot,k]\cdot \M}+\sqrt{\C[\bigcdot,k]\cdot \N}$ satisfies the necessary conditions of Proposition~\ref{prop_nec}.
%\end{prop}

\begin{proof}
We show that $\lambda_k=\sqrt{\R[\bigcdot,k]\cdot \M}+\sqrt{\C[\bigcdot,k]\cdot \N}$ satisfies the necessary conditions of Proposition~\ref{prop_nec}. For condition 1., note that 
$\sqrt{\R[\bigcdot,k]\cdot \M}+\sqrt{\C[\bigcdot,k]\cdot \N}>\sqrt{\R[\bigcdot,j]\cdot\M}+\sqrt{\C[\bigcdot,j]\cdot\N}$.

For condition 2., note that
\begin{align*}
\sum_{j\in\mathcal{I}_k} \sqrt{\R[\bigcdot,j]\cdot\M}+\sqrt{\C[\bigcdot,j]\cdot\N}&\geq \sqrt{ \sum_{j\in\mathcal{I}_k} \R[\bigcdot, j]\cdot \M}+\sqrt{\sum_{j\in\mathcal{I}_k} \C[\bigcdot,j]\cdot\N}\\
&=\sqrt{r\cdot \R[\bigcdot,k]\cdot \M}+\sqrt{c\cdot \C[\bigcdot, k]\cdot \N}\\
&> \sqrt{\R[\bigcdot,k]\cdot \M}+\sqrt{\C[\bigcdot, k]\cdot \N}    
\end{align*}
\end{proof}

\subsection{Proof of Proposition 5}

Lemmas~\ref{lem1} and \ref{lem2} below are used to establish Proposition 5.

\begin{lemma} \label{lem1} Take two decompositions  $\SolHat \in \SolSpace$ and $\SolTil \in \SolSpace$, and assume that both minimize the structured nuclear norm penalty: 
\[\fpen (\SolHat)=\fpen \left(\SolTil\right)=\underset{\SolSpace}{\min} \; \fpen  (\Sol).\]
Then, for any $\alpha \in [0,1]$,
\begin{align*}
||\alpha \hatSk+(1-\alpha)\tilSk||_* = \alpha ||\hatSk||_*+(1-\alpha) ||\tilSk||_*  
\end{align*}
for $k=1,\hdots,K$.
\end{lemma} 

\begin{proof}
Because $\SolSpace$ is a convex space and $\fpen$ is a convex function, the set of minimizers of $\fpen$  over $\SolSpace$ is also convex.  Thus, 
\[\fpen \left(\{\alpha \hatSk + (1-\alpha) \tilSk\}_{k=1}^K \right)=\underset{\SolSpace}{\min} \; \fpen  (\Sol).\]
The result follows from the convex property of the nuclear norm operator, which implies that for any two matrices of equal size $\hat{\A}$ and $\tilde{\A}$, 
\begin{align}||\alpha \hat{\A} + (1-\alpha) \tilde{\A}||_* \leq \alpha ||\hat{\A}||_* + (1-\alpha)||\tilde{\A}||_*. \label{pen_conv} \end{align} 
  Applying~\eqref{pen_conv} to each additive term  in $\fpen$ gives 
\begin{align}
 \fpen \left(\{\alpha \hatSk + (1-\alpha) \tilSk\}_{k=1}^K \right) &\leq \alpha \fpen(\SolHat) + (1-\alpha) \fpen(\SolTil) \label{ineq1} \\
 &=  \underset{\SolSpace}{\min} \; \fpen  (\Sol). \notag 
 \end{align}
Because $\{\alpha \hatSk + (1-\alpha) \tilSk\}_{k=1}^K \in \SolSpace$, the inequality  in~\eqref{ineq1} must be an equality,  and it follows that the inequality~\eqref{pen_conv} must be an equality for each penalized term in the decomposition.   
\end{proof}

\begin{lemma}
\label{lem2}
Take two matrices $\hat{\A}$ and $\tilde{\A}$. If $||\hat{\A}+\tilde{\A}||_*=||\hat{\A}||_*+||\tilde{\A}||_*$, and $\U \D_+ \V^T$ is the SVD of $\hat{\A}+\tilde{\A}$, then $\hat{\A} = \hat{\U} \hat{\D} \hat{\V}^T$ where $\hat{\D}$ is diagonal and $||\hat{\A}||_*=||\hat{\D}||_*$, and $\tilde{\A} = \U \tilde{\D} \V^T$ where $\tilde{\D}$ is diagonal and $||\tilde{\A}||_*=||\tilde{\D}||_*$. 
\end{lemma}
\begin{proof}
Here we use the fact that the spectral norm is dual to the nuclear norm \citep{fazel2001rank}. That is, if $\sigma_1(\Z)$ is the maximum singular value of $\Z$ (i.e., the spectral norm), then 
\[||\A||_* = \underset{\sigma_1(\Z)=1}{\mbox{sup}}\langle \Z,\A \rangle.\]
Thus,
\begin{align} \label{aprop1} \underset{\sigma_1(\Z)=1}{\mbox{sup}} \langle \Z,\hat{\A}+\tilde{\A} \rangle = \underset{\sigma_1(\Z)=1}{\mbox{sup}} \langle \Z,\hat{\A} \rangle + \underset{\sigma_1(\Z)=1}{\mbox{sup}} \langle \Z,\tilde{\A} \rangle.\end{align}
By the properties of the SVD,
\begin{align} \label{aprop2} \langle \U \V^T,\tilde{\A} \rangle+ \langle \U \V^T, \hat{\A} \rangle= \langle \U \V^T, \hat{\A}+\tilde{\A} \rangle = \underset{\sigma_1(\Z)=1}{\mbox{sup}} \langle \Z,\hat{\A}+\tilde{\A} \rangle.\end{align}
By~\eqref{aprop1} and~\eqref{aprop2}, 
\begin{align}\langle \U \V^T,\tilde{\A} \rangle+ \langle \U \V^T, \hat{\A} \rangle = \underset{\sigma_1(\Z)=1}{\mbox{sup}} \langle \Z,\hat{\A} \rangle + \underset{\sigma_1(\Z)=1}{\mbox{sup}} \langle \Z,\tilde{\A} \rangle. \label{eqsum} \end{align}
Because $\U \V^T \in \{\Z: \sigma_1(\Z)=1\}$ it follows that 
\[\langle \U \V^T,\hat{\A}\rangle \leq  \underset{\sigma_1(\Z)=1}{\mbox{sup}} \langle \Z,\hat{\A} \rangle \; \; \text{and} \; \; \langle \U \V^T,\tilde{\A}\rangle\leq\underset{\sigma_1(\Z)=1}{\mbox{sup}}\langle\Z,\tilde{\A}\rangle,\]
and so \eqref{eqsum} implies  
\[\langle \U \V^T,\hat{\A} \rangle = \underset{\sigma_1(\Z)=1}{\mbox{sup}}\langle\Z,\hat{\A}\rangle = ||\hat{\A}||_*,\]
and similarly $\langle \U \V^T,\tilde{\A} \rangle =  ||\tilde{\A}||_*$.  

Let $\tilde{\U} \tilde{\D} \tilde{\V}^T$ give the SVD of $\tilde{\A}$. Note that 
\[\langle \U \V^T, \tilde{\U} \tilde{\D} \tilde{\V}^T\rangle = \mbox{Tr}(\V \U^T \tilde{\U} \tilde{\D} \tilde{\V}^T ) = \mbox{Tr}(\V^T \tilde{\V} \U^T \tilde{\U} \tilde{\D}),\]
and \[\mbox{Tr}(\V^T \tilde{\V} \U^T \tilde{\U} \tilde{\D}) = ||\tilde{\A}||_*= \sum_i \tilde{\D}[i,i]\]
if and only if $\V^T \tilde{\V} \U^T \tilde{\U} [i,i] = 1$ where $\tilde{\D}[i,i]>0$.  It follows that the left and right singular vectors of $\tilde{\A}$ that correspond to non-zero singular values must also be singular vectors of $\hat{\A}+\tilde{\A}$.  By an analogous argument, the left and right singular vectors that correspond to non-zero singular values in $\hat{\A}$ must also be singular vectors of $\hat{\A}+\tilde{\A}$.  
\end{proof}

%\begin{prop}
%\label{propgen}
%Take two decompositions  $\SolHat \in \SolSpace$ and $\SolTil \in \SolSpace$, and assume that both minimize the structured nuclear norm penalty: 
%\[\fpen (\SolHat)=\fpen \left(\SolTil\right)=\underset{\SolSpace}{\min} \; \fpen  (\Sol).\] Then, $\hatSk = \U_{\bigcdot}^{(k)} \hat{\D}^{(k)} \V_{\bigcdot}^{(k) T}$ and $\hatSk = \U_{\bigcdot}^{(k)} \tilde{\D}^{(k)} \V_{\bigcdot}^{(k) T}$ where 
%$\U_{\bigcdot}^{(k)}: M \times R_k$ and $\V_{\bigcdot}^{(k)}: N \times R_k$ have orthonormal columns, and
%$\hat{\D}^{(k)}$ and $\tilde{\D}^{(k)}$ are diagonal. %, and 
%\item if $\hat{\D}^{(k)}[r,r] \neq \tilde{\D}^{(k)}[r,r]$, then one of $\hat{\D}^{(k)}[r,r]$ or $\tilde{\D}^{(k)}[r,r]$ is 0.  
%\end{prop}

Proposition 5 is a direct corollary of Lemmas \ref{lem1} and \ref{lem2},  as Lemma~\ref{lem1} implies $||\hatSk+\tilSk||_*=||\hatSk||_*+||\tilSk||_*$ for each $k$, and then Lemma~\ref{lem2} 
implies the result. 	

\subsection{Proof of Theorem 1}

%\begin{theorem}
%\label{identTheor}  Consider $\SolHat \in \SolSpace$ and let $\U_{\bigcdot}^{(k)} \hat{\D} \V_{\bigcdot}^{(k) T}$ give the SVD of $\hatSk$ for $k=1,\hdots,K$. The following three properties uniquely identify $\SolHat$.\begin{enumerate}
%\item $\SolHat$ minimizes $\fpen (\cdot)$ over $\SolSpace$,   
%\item $\{\hat{\U}_i^{(k)}[\bigcdot,r]: \R[i,k] = 1 \text{ and } \hat{\D}^{(k)}[r,r]>0\}$ are linearly independent for $i=1,\hdots I$,  
%\item $\{\hat{\V}_j^{(k)}[\bigcdot,r]: \C[j,k] = 1 \text{ and } \hat{\D}^{(k)}[r,r]>0\}$ are linearly independent for $j=1,\hdots,J$.
%\end{enumerate} 
%\end{theorem}
\begin{proof}
Take two decomposition $\SolHat$ and $\SolTil$ that satisfy properties 1., 2., and 3. of Theorem 1; we will show that $\SolHat=\SolTil$. For each $k=1,\hdots,K$, write $\hatSk = \U_{\bigcdot}^{(k)} \hat{\D} \V_{\bigcdot}^{(k) T}$ and $\hatSk = \U_{\bigcdot}^{(k)} \tilde{\D}^{(k)} \V_{\bigcdot}^{(k) T}$ as in Proposition~\ref{propgen}.   Then, it suffices to show that $\hat{\D}^{(k)}[r,r]=\tilde{\D}^{(k)}[r,r]$ for all $k,r$.  

First, consider module $k=1$ with $\R[\bigcdot,1]=[1 \;0 \;\cdots\; 0]^T$ and $\C[\bigcdot,1]=[1 \; 0 \; \cdots \;0]^T$.    By way of contradiction, assume $\hat{\D}^{(1)}[1,1]>0$ and $\tilde{\D}^{(1)}[1,1]=0$.  The linear independence of $\{\V_j^{(k)}[\bigcdot,r]: \hat{\D}^{(k)}[r,r]>0\}$ and $\{\V_j^{(k)}[\bigcdot,r]: \tilde{\D}^{(k)}[r,r]>0\}$ implies that
\[\mbox{row}(\X_{\bigcdot \bigcdot}) = \mbox{span}\{\U_{\bigcdot}^{(k)}[\bigcdot,r]: \hat{\D}^{(k)}[r,r]>0\}=\mbox{span}\{\{\U_{\bigcdot}^{(k)}[\bigcdot,r]: \tilde{\D}^{(k)}[r,r]>0\}.\]
Thus, $\U^{(1)}[\bigcdot,1]] \in \mbox{span}\{\{\U_{\bigcdot}^{(k)}[\bigcdot,r]: \tilde{\D}^{(k)}[r,r]>0\}$, and it follows from the orthogonality of $\U^{(1)}[\bigcdot,1]$ and $\{\U^{(1)}[\bigcdot,r], r>1\}$ that \[\U_{\bigcdot}^{(1)}[\bigcdot,1] \in \mbox{span}\{\{\U_{\bigcdot}^{(k)}[\bigcdot,r]: \tilde{\D}^{(k)}[r,r]>0 \text{ and }  k>1\}.\]
Moreover, because $\U_{i}^{(1)}=\0$ for any $i>1$ and $\{\U_i^{(k)}[\bigcdot,r]: \tilde{\D}^{(k)}[r,r]>0\}$ are linearly independent it follows that
   \begin{align} \U_{\bigcdot}^{(1)}[\bigcdot,1] \in \mbox{span}\{\U_{\bigcdot}^{(k)}[\bigcdot,r]: \tilde{\D}^{(k)}[r,r]>0, \; k>1, \;  \text{ and } \R[i,k]=0 \text{ for any } i>1\}.\label{temp1}\end{align}
Note that \eqref{temp1} implies $\U_{1}^{(1)}[\bigcdot,1] \in \mbox{row}(\X_{12} + \cdots + \mbox{row}(\X_{1J})$,
however, this is contradicted by the linear independence of $\U_{1}^{(1)}[\bigcdot,1]$ and $\{\U_i^{(k)}[\bigcdot,r]: \hat{\D}^{(k)}[r,r]>0, k>1\}$.  Thus, we conclude that $\tilde{\D}^{(1)}[1,1]>0$ implies $\tilde{\D}^{(1)}[1,1]>0$.  Analogous arguments show that $\tilde{\D}^{(k)}[r,r]>0$ if and only if $\tilde{\D}^{(k)}[r,r]>0$ for any pair $(r,k)$.  It follows that  $\{\U_i^{(k)}[\bigcdot,r]: \hat{\D}^{(k)}[r,r]>0 \text{ or } \tilde{\D}^{(k)}[r,r]>0\}$ are linearly independent for $i=1,\hdots I$, and $\{\V_j^{(k)}[\bigcdot,r]: \hat{\D}^{(k)}[r,r]>0 \text{ or } \tilde{\D}^{(k)}[r,r]>0\}$ are linearly independent for $j=1,\hdots,J$.  Thus,   
\begin{align*}
\sum_{k=1}^K \U_{\bigcdot}^{(k)} (\hat{\D}^{(k)}-\tilde{\D}^{(k)}) \V_{\bigcdot}^{(k) T} = \sum_{k=1}^K \hatSk - \tilSk = \X_{\bigcdot \bigcdot} - \X_{\bigcdot \bigcdot}=\0 	
\end{align*}

implies that $\hat{\D}^{(k)}[r,r]=\tilde{\D}^{(k)}[r,r]$ for all $k,r$.  

\end{proof}

\bibliographystyle{apa}

\bibliography{bibliography.bib}

\begin{thebibliography}{}

\bibitem[\protect\astroncite{Akbani et~al.}{2014}]{akbani2014pan}
Akbani, R., Ng, P. K.~S., Werner, H.~M., Shahmoradgoli, M., Zhang, F., Ju, Z.,
  Liu, W., Yang, J.-Y., Yoshihara, K., Li, J., et~al. (2014).
\newblock A pan-cancer proteomic perspective on the cancer genome atlas.
\newblock {\em Nature communications}, 5:3887.

\bibitem[\protect\astroncite{Argelaguet et~al.}{2018}]{argelaguet2018multi}
Argelaguet, R., Velten, B., Arnol, D., Dietrich, S., Zenz, T., Marioni, J.~C.,
  Buettner, F., Huber, W., and Stegle, O. (2018).
\newblock Multi-omics factor analysis-a framework for unsupervised integration
  of multi-omics data sets.
\newblock {\em Molecular systems biology}, 14(6).

\bibitem[\protect\astroncite{Fazel et~al.}{2001}]{fazel2001rank}
Fazel, M., Hindi, H., Boyd, S.~P., et~al. (2001).
\newblock A rank minimization heuristic with application to minimum order
  system approximation.
\newblock In {\em Proceedings of the American control conference}, volume~6,
  pages 4734--4739. Citeseer.

\bibitem[\protect\astroncite{Gabasova et~al.}{2017}]{gabasova2017clusternomics}
Gabasova, E., Reid, J., and Wernisch, L. (2017).
\newblock Clusternomics: Integrative context-dependent clustering for
  heterogeneous datasets.
\newblock {\em PLoS computational biology}, 13(10):e1005781.

\bibitem[\protect\astroncite{Gavish and Donoho}{2017}]{gavish2017optimal}
Gavish, M. and Donoho, D.~L. (2017).
\newblock Optimal shrinkage of singular values.
\newblock {\em IEEE Transactions on Information Theory}, 63(4):2137--2152.

\bibitem[\protect\astroncite{Gaynanova and Li}{2019}]{gaynanova2017structural}
Gaynanova, I. and Li, G. (2019).
\newblock Structural learning and integrative decomposition of multi-view data.
\newblock {\em Biometrics}.

\bibitem[\protect\astroncite{Hellton and
  Thoresen}{2016}]{hellton2016integrative}
Hellton, K.~H. and Thoresen, M. (2016).
\newblock Integrative clustering of high-dimensional data with joint and
  individual clusters.
\newblock {\em Biostatistics}, 17(3):537--548.

\bibitem[\protect\astroncite{Hoadley et~al.}{2018}]{hoadley2018cell}
Hoadley, K.~A., Yau, C., Hinoue, T., Wolf, D.~M., Lazar, A.~J., Drill, E.,
  Shen, R., Taylor, A.~M., Cherniack, A.~D., Thorsson, V., et~al. (2018).
\newblock Cell-of-origin patterns dominate the molecular classification of
  10,000 tumors from 33 types of cancer.
\newblock {\em Cell}, 173(2):291--304.

\bibitem[\protect\astroncite{Hoadley et~al.}{2014}]{hoadley2014multiplatform}
Hoadley, K.~A., Yau, C., Wolf, D.~M., Cherniack, A.~D., Tamborero, D., Ng, S.,
  Leiserson, M.~D., Niu, B., McLellan, M.~D., Uzunangelov, V., et~al. (2014).
\newblock Multiplatform analysis of 12 cancer types reveals molecular
  classification within and across tissues of origin.
\newblock {\em Cell}, 158(4):929--944.

\bibitem[\protect\astroncite{Huo and Tseng}{2017}]{huo2017integrative}
Huo, Z. and Tseng, G. (2017).
\newblock Integrative sparse k-means with overlapping group lasso in genomic
  applications for disease subtype discovery.
\newblock {\em The annals of applied statistics}, 11(2):1011.

\bibitem[\protect\astroncite{Hutter and Zenklusen}{2018}]{hutter2018cancer}
Hutter, C. and Zenklusen, J.~C. (2018).
\newblock {The Cancer Genome Atlas}: creating lasting value beyond its data.
\newblock {\em Cell}, 173(2):283--285.

\bibitem[\protect\astroncite{Kandoth et~al.}{2013}]{kandoth2013mutational}
Kandoth, C., McLellan, M.~D., Vandin, F., Ye, K., Niu, B., Lu, C., Xie, M.,
  Zhang, Q., McMichael, J.~F., Wyczalkowski, M.~A., et~al. (2013).
\newblock Mutational landscape and significance across 12 major cancer types.
\newblock {\em Nature}, 502(7471):333.

\bibitem[\protect\astroncite{Kaplan and Lock}{2017}]{kaplan2017prediction}
Kaplan, A. and Lock, E.~F. (2017).
\newblock Prediction with dimension reduction of multiple molecular data
  sources for patient survival.
\newblock {\em Cancer Informatics}, 16:1--11.

\bibitem[\protect\astroncite{Kurucz et~al.}{2007}]{kurucz2007methods}
Kurucz, M., Bencz{\'u}r, A.~A., and Csalog{\'a}ny, K. (2007).
\newblock Methods for large scale svd with missing values.
\newblock In {\em Proceedings of KDD cup and workshop}, volume~12, pages
  31--38.

\bibitem[\protect\astroncite{Li et~al.}{2018}]{li2018general}
Li, G., Gaynanova, I., et~al. (2018).
\newblock A general framework for association analysis of heterogeneous data.
\newblock {\em The Annals of Applied Statistics}, 12(3):1700--1726.

\bibitem[\protect\astroncite{Li and Jung}{2017}]{li2017incorporating}
Li, G. and Jung, S. (2017).
\newblock Incorporating covariates into integrated factor analysis of
  multi-view data.
\newblock {\em Biometrics}, 73(4):1433--1442.

\bibitem[\protect\astroncite{Lock and Dunson}{2013}]{lock2013bayesian}
Lock, E.~F. and Dunson, D.~B. (2013).
\newblock Bayesian consensus clustering.
\newblock {\em Bioinformatics}, 29(20):2610--2616.

\bibitem[\protect\astroncite{Lock et~al.}{2013}]{lock2013joint}
Lock, E.~F., Hoadley, K.~A., Marron, J., and Nobel, A.~B. (2013).
\newblock Joint and {I}ndividual {V}ariation {E}xplained ({JIVE}) for
  integrated analysis of multiple data types.
\newblock {\em The Annals of Applied Statistics}, 7(1):523.

\bibitem[\protect\astroncite{Mazumder et~al.}{2010}]{mazumder2010spectral}
Mazumder, R., Hastie, T., and Tibshirani, R. (2010).
\newblock Spectral regularization algorithms for learning large incomplete
  matrices.
\newblock {\em Journal of machine learning research}, 11(Aug):2287--2322.

\bibitem[\protect\astroncite{Mo et~al.}{2017}]{mo2017fully}
Mo, Q., Shen, R., Guo, C., Vannucci, M., Chan, K.~S., and Hilsenbeck, S.~G.
  (2017).
\newblock A fully bayesian latent variable model for integrative clustering
  analysis of multi-type omics data.
\newblock {\em Biostatistics}, 19(1):71--86.

\bibitem[\protect\astroncite{Noushmehr
  et~al.}{2010}]{noushmehr2010identification}
Noushmehr, H., Weisenberger, D.~J., Diefes, K., Phillips, H.~S., Pujara, K.,
  Berman, B.~P., Pan, F., Pelloski, C.~E., Sulman, E.~P., Bhat, K.~P., et~al.
  (2010).
\newblock Identification of a cpg island methylator phenotype that defines a
  distinct subgroup of glioma.
\newblock {\em Cancer cell}, 17(5):510--522.

\bibitem[\protect\astroncite{O'Connell and Lock}{2016}]{oconnell2016}
O'Connell, M.~J. and Lock, E.~F. (2016).
\newblock R. jive for exploration of multi-source molecular data.
\newblock {\em Bioinformatics}, 32(18):2877--2879.

\bibitem[\protect\astroncite{O'Connell and Lock}{2019}]{o2019linked}
O'Connell, M.~J. and Lock, E.~F. (2019).
\newblock Linked matrix factorization.
\newblock {\em Biometrics}, 75(2):582--592.

\bibitem[\protect\astroncite{Park and Lock}{2020}]{park2019integrative}
Park, J.~Y. and Lock, E.~F. (2020).
\newblock Integrative factorization of bidimensionally linked matrices.
\newblock {\em Biometrics}, 76(1):61--74.

\bibitem[\protect\astroncite{Rudelson and Vershynin}{2010}]{rudelson2010non}
Rudelson, M. and Vershynin, R. (2010).
\newblock Non-asymptotic theory of random matrices: extreme singular values.
\newblock In {\em Proceedings of the International Congress of Mathematicians
  2010 (ICM 2010) (In 4 Volumes) Vol. I: Plenary Lectures and Ceremonies Vols.
  II--IV: Invited Lectures}, pages 1576--1602. World Scientific.

\bibitem[\protect\astroncite{Shabalin and
  Nobel}{2013}]{shabalin2013reconstruction}
Shabalin, A.~A. and Nobel, A.~B. (2013).
\newblock Reconstruction of a low-rank matrix in the presence of gaussian
  noise.
\newblock {\em Journal of Multivariate Analysis}, 118:67--76.

\bibitem[\protect\astroncite{Shen et~al.}{2013}]{shen2013sparse}
Shen, R., Wang, S., and Mo, Q. (2013).
\newblock Sparse integrative clustering of multiple omics data sets.
\newblock {\em The annals of applied statistics}, 7(1):269.

\bibitem[\protect\astroncite{{TCGA Research
  Network}}{2015}]{cancer2015comprehensive}
{TCGA Research Network} (2015).
\newblock {Comprehensive, integrative genomic analysis of diffuse lower-grade
  gliomas}.
\newblock {\em New England Journal of Medicine}, 372(26):2481--2498.

\bibitem[\protect\astroncite{{TCGA Research Network}
  et~al.}{2012}]{cancer2012comprehensive}
{TCGA Research Network} et~al. (2012).
\newblock Comprehensive molecular portraits of human breast tumors.
\newblock {\em Nature}, 490(7418):61.

\bibitem[\protect\astroncite{{TCGA Research Network}
  et~al.}{2014}]{cancer2014comprehensive}
{TCGA Research Network} et~al. (2014).
\newblock Comprehensive molecular profiling of lung adenocarcinoma.
\newblock {\em Nature}, 511(7511):543.

\bibitem[\protect\astroncite{Verhaak et~al.}{2010}]{verhaak2010integrated}
Verhaak, R.~G., Hoadley, K.~A., Purdom, E., Wang, V., Qi, Y., Wilkerson, M.~D.,
  Miller, C.~R., Ding, L., Golub, T., Mesirov, J.~P., et~al. (2010).
\newblock Integrated genomic analysis identifies clinically relevant subtypes
  of glioblastoma characterized by abnormalities in {PDGFRA, IDH1, EGFR, and
  NF1}.
\newblock {\em Cancer cell}, 17(1):98--110.

\bibitem[\protect\astroncite{Weinstein et~al.}{2013}]{weinstein2013cancer}
Weinstein, J.~N., Collisson, E.~A., Mills, G.~B., Shaw, K. R.~M., Ozenberger,
  B.~A., Ellrott, K., Shmulevich, I., Sander, C., Stuart, J.~M., Network, C. G.
  A.~R., et~al. (2013).
\newblock The cancer genome atlas pan-cancer analysis project.
\newblock {\em Nature genetics}, 45(10):1113--1120.

\bibitem[\protect\astroncite{Yang and Michailidis}{2016}]{yang2015non}
Yang, Z. and Michailidis, G. (2016).
\newblock A non-negative matrix factorization method for detecting modules in
  heterogeneous omics multi-modal data.
\newblock {\em Bioinformatics}, 32(1):1--8.

\bibitem[\protect\astroncite{Zack et~al.}{2013}]{zack2013pan}
Zack, T.~I., Schumacher, S.~E., Carter, S.~L., Cherniack, A.~D., Saksena, G.,
  Tabak, B., Lawrence, M.~S., Zhang, C.-Z., Wala, J., Mermel, C.~H., et~al.
  (2013).
\newblock Pan-cancer patterns of somatic copy number alteration.
\newblock {\em Nature genetics}, 45(10):1134--1140.

\bibitem[\protect\astroncite{Zhang et~al.}{2005}]{zhang2005using}
Zhang, S., Wang, W., Ford, J., Makedon, F., and Pearlman, J. (2005).
\newblock Using singular value decomposition approximation for collaborative
  filtering.
\newblock In {\em Seventh IEEE International Conference on E-Commerce
  Technology (CEC'05)}, pages 257--264. IEEE.

\bibitem[\protect\astroncite{Zhu et~al.}{2020}]{zhu2018generalized}
Zhu, H., Li, G., and Lock, E.~F. (2020).
\newblock Generalized integrative principal component analysis for multi-type
  data with block-wise missing structure.
\newblock {\em Biostatistics}, 21(2):302--318.

\end{thebibliography}
\end{document}